\documentclass{article}

\PassOptionsToPackage{numbers,compress}{natbib}

\usepackage[main,final]{neurips_2026}

\usepackage[utf8]{inputenc} % allow utf-8 input
\usepackage[T1]{fontenc}    % use 8-bit T1 fonts
\usepackage{hyperref}       % hyperlinks
\hypersetup{hidelinks}
\usepackage{url}            % simple URL typesetting
\usepackage{booktabs}       % professional-quality tables
\usepackage{amsfonts}       % blackboard math symbols
\usepackage{nicefrac}       % compact symbols for 1/2, etc.
\usepackage{microtype}      % microtypography
\usepackage{xcolor}         % colors
\usepackage{multirow}
\usepackage{graphicx}
\usepackage{enumitem}
\usepackage{bm,tikz,bbm}
\usepackage{colortbl}
\usepackage{amsmath, amssymb, amsthm}

\graphicspath{{figures/}}

\theoremstyle{plain}
\newtheorem{theorem}{Theorem}[section]

\newtheorem{lemma}[theorem]{Lemma}
\newtheorem{corollary}[theorem]{Corollary}
\theoremstyle{definition}

\theoremstyle{remark}

\newcommand*\del[0]{\partial}
\renewcommand*\div[0]{\textbf{div}}
\newcommand*\lin[1]{\bm{\left\langle} #1 \bm{\right\rangle}}
\newcommand*\E[1]{\mathbb{E}\left[{#1}\right]}
\newcommand*\Ep[2]{\mathbb{E}_{{#1}}\left[{#2}\right]}
\newcommand*\N[0]{\mathcal{N}}
\newcommand\numberthis{\addtocounter{equation}{1}\tag{\theequation}}

\newcommand*\lrp[1]{{\left(#1\right)}}
\newcommand*\lrn[1]{{\left\|#1\right\|}}

\renewcommand*{\Re}{\mathbb{R}}

\newcommand*{\red}[1]{{\textcolor{red}{#1}}}
\newcommand*{\blue}[1]{{\textcolor{blue}{#1}}}

\definecolor{darkgreen}{rgb}{0,0.5,0}

\newcommand{\net}{\mathrm{net}}
\newcommand{\flow}{\mathrm{flow}}

\newcommand{\fidstd}[2]{\begin{tabular}{@{}c@{}}#1\\[-0.15em]{\scriptsize\textcolor{gray}{(#2)}}\end{tabular}}
\newcommand{\px}[1]{}
\newcommand{\michael}[1]{}


\title{Variational Trajectory Optimization of Anisotropic Diffusion Schedules}

\author{%
  Pengxi Liu\thanks{Equal contribution.}
  \quad
  Zeyu Michael Li\footnotemark[1]
  \quad
  Xiang Cheng \\
  Department of Electrical and Computer Engineering \\
  Duke University \\
  \texttt{\{pengxi.liu,zeyu.li030,xiang.cheng\}@duke.edu}
}
\begin{document}

\maketitle

\begin{abstract}
We introduce a variational framework for diffusion models with anisotropic noise schedules parameterized by a matrix-valued path $M_t(\theta)$ that allocates noise across subspaces. Central to our framework is a trajectory-level objective that jointly trains the score network and \emph{learns} $M_t(\theta)$, which encompasses general parameterization classes of matrix-valued noise schedules. We further derive an estimator for $\partial_\theta \nabla \log p_t$ that enables efficient optimization of the $M_t(\theta)$ schedule. For inference, we develop an efficiently-implementable reverse-ODE solver that is an anisotropic generalization of the second-order Heun discretization algorithm. Across CIFAR-10, AFHQv2, FFHQ, and ImageNet-64, our method consistently improves FID with largest gains at low NFE. Code is available at \url{https://github.com/kilonova801/anisotropic-diffusion}.
\end{abstract}

\section{Introduction}

Diffusion models generate samples by \emph{reversing} a gradual Gaussian noising process. Starting from data $x_0\sim p_0$, a forward process produces a family of noisy variables $(x_t)_{t\ge 0}$ with density $p_t$ that smoothly interpolates between the data distribution and a tractable Gaussian. A neural \emph{score} model (parameters $\phi$) is trained to approximate the score $\nabla_x\log p_t(x)$; given this score, sampling integrates a reverse-time SDE/ODE that transports $p_T$ back to $p_0$ \citep{ho2020denoising,song2021score_sde,karras2022elucidating}. Nearly all modern diffusion pipelines assume \emph{isotropic} forward noise, meaning that at time $t$, the covariance of total injected Gaussian noise is a scalar multiple of identity. We ask whether learning separate noise schedules for different data subspaces, jointly with the score network, can improve generation, especially under limited sampling budgets.

\textbf{Matrix-valued (anisotropic) schedules.} In this paper, we replace the scalar noise level by a \emph{matrix-valued} covariance trajectory $M_t$ and study how to \emph{learn} it. Concretely, in the standard variance-exploding (VE) / Brownian formulation, the isotropic process adds noise with covariance $tI$,
\begin{equation}
\begin{aligned}
& x_0\!\sim\!p_0,\quad d x_t=dB_t \quad
\Leftrightarrow \quad
 d x_t=-\tfrac12\nabla\log p_t(x_t)dt,
\label{e:iso-ve}
\end{aligned}
\end{equation}
where $B_t$ is standard Brownian motion. We generalize this to anisotropic diffusion by instead considering a forward diffusion driven by a \emph{matrix-valued diffusion coefficient}:
\begin{equation}
\begin{aligned}
& x_0\!\sim\!p_0,\quad d x_t=(\partial_t M_t)^{1/2}dB_t \quad \Leftrightarrow \quad
 d x_t=-\tfrac12\partial_t M_t\nabla\log p_t(x_t)dt,
\label{e:anisotropic_diffusion}
\end{aligned}
\end{equation}
where $M_t(\theta)$ is a matrix-valued noise-schedule, with arbitrary parameterization $\theta$, satisfying
\begin{equation}
  M_0=0,
  \qquad
  \partial_t M_t\succ 0 \qquad \text{for all } t.
  \label{e:Mt-psd}
\end{equation}

For example, let $P_{\mathrm{low}}$ and $P_{\mathrm{high}}$
be orthogonal projectors onto complementary low- and
high-frequency subspaces.
Then
\[
M_t
=
g_{\mathrm{low}}(t)P_{\mathrm{low}}
+
g_{\mathrm{high}}(t)P_{\mathrm{high}}
\]
assigns separate noise variances to the two subspaces.
When $g_{\mathrm{low}}(t)=g_{\mathrm{high}}(t)$,
this reduces to an isotropic schedule.
Our approach learns how these noise variances evolve over time.

\textbf{The power of anisotropy.} Replacing a scalar schedule by a matrix path allows noise (and thus denoising effort) to be allocated differently across directions and time, better matching data geometry. Natural images concentrate energy in low spatial frequencies \citep{ruderman1993statistics}; latent diffusion effectively separates coarse structure from fine detail \citep{rombach2022high}; video models benefit from temporally structured priors or decomposed noise \citep{ge2023preserve,luo2023videofusion}; and multi-resolution generation often proceeds coarse-to-fine \citep{tian2024visual}. Figure~\ref{fig:visulize_v_h} illustrates one such effect: low-frequency content can be denoised earlier than high-frequency details when the schedule is allowed to differ across subspaces.

\textbf{A variational formulation of matrix trajectory is necessary.} The flexibility of $M_t$ comes with a price: the design space is enormous. Instead of choosing a single scalar function, we must choose a trajectory in the cone of PSD matrices. In the isotropic VE case, the family of marginals is essentially one-dimensional: $p_t = p_0*\mathcal N(0,\sigma(t)^2 I)$ (where $*$ denotes convolution) depends on $t$ only through the scalar variance $\sigma(t)^2$, so changing the schedule mainly \emph{reparameterizes time} along the same curve of distributions. In contrast, different anisotropic paths $M_t(\theta)$ generally produce genuinely different marginals $p_t(\cdot;\theta)$ (even up to time reparameterization), and hence different score targets and reverse-time dynamics. This makes hand-crafted anisotropy brittle and motivates a principled, data-driven approach.

\begin{figure*}
    \centering
    \includegraphics[width=\linewidth]{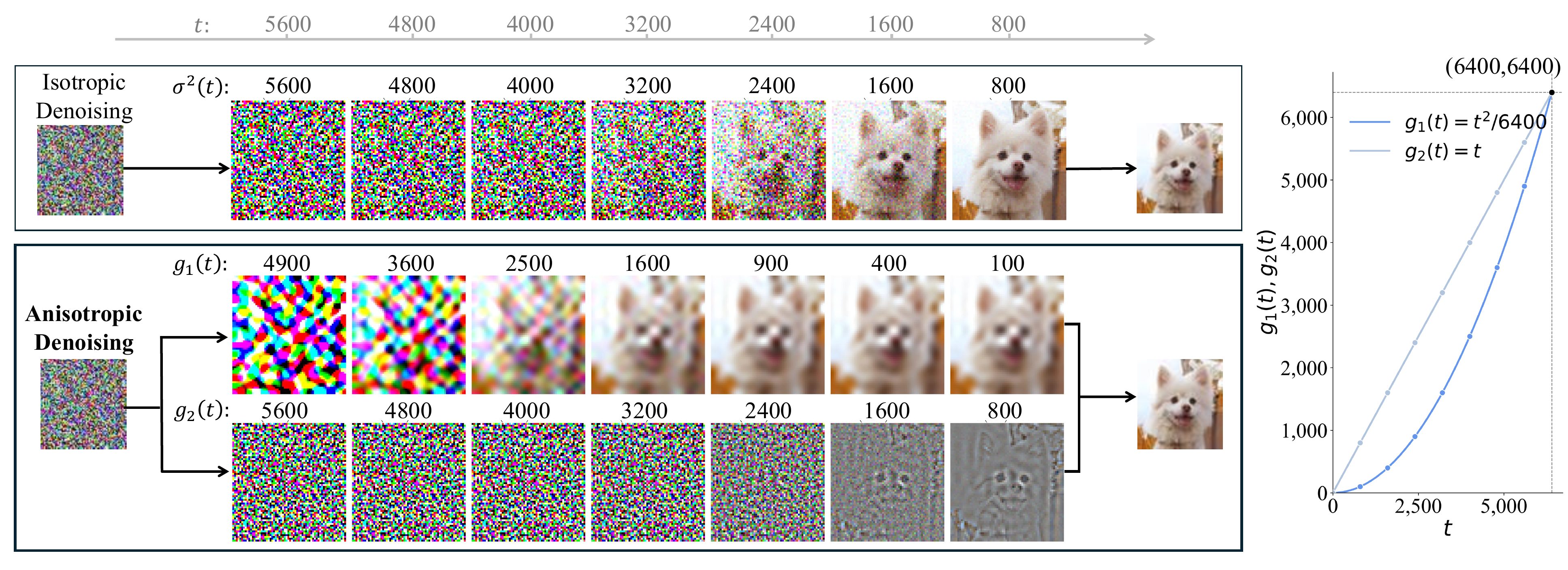}
    \caption{
    \textbf{Isotropic vs.\ anisotropic denoising.}
    Top: isotropic denoising.
    Bottom: anisotropic denoising with low- and high-frequency
    DCT subspaces (Section~\ref{s:implementation_details}).
    Columns schematically illustrate denoising as $t$ decreases.
    Right: schematic schedules $g_1,g_2$ for the low- and
    high-frequency subspaces, respectively, illustrating
    how low-frequency structure can emerge before
    high-frequency detail.
    \textbf{Illustration only:}
    Both rows use the same dog image from
    AFHQv2~\citep{choi2020stargan}
    to illustrate the different denoising patterns.
    The $g_1$--$g_2$ gap is exaggerated for clarity;
    the learned schedules exhibit smaller differences
    (cf.\ Fig.~\ref{fig:cifar10_schedule}--\ref{fig:imagenet_schedule}).
    }
    \label{fig:visulize_v_h}
\end{figure*}
\subsection{Main Contributions}
We develop a variational framework for \emph{learning} anisotropic diffusion trajectories $M_t(\theta)$ jointly with a score model (parameters $\phi$), rather than prescribing $M_t$ by hand. We summarize our technical and empirical contributions below.

\paragraph{(1) A general variational framework for learning anisotropic diffusion models.}
\begin{enumerate}[leftmargin=*,nosep]
    \item We introduce a variational framework for anisotropic diffusion with \emph{matrix-valued} noise path $M_t(\theta)$. Our framework is general and supports broad classes of $M_t(\theta)$, including interpretable frequency-band schedules and data-adaptive (class-conditional) PCA schedules (Section~\ref{s:implementation_details}).
    \item We propose a trajectory-level score-matching objective $L(\theta,\phi)$ (Section~\ref{sec:pathwise-objective}) that jointly trains the score network and learns $M_t(\theta)$ by minimizing mismatch between ideal and learned denoising dynamics along the reverse trajectory (with a pathwise change-of-measure interpretation).
    \item We derive anisotropic reverse-ODE solvers by generalizing Euler and second-order Heun updates to matrix trajectories (Section~\ref{sec:inference_matrix}), yielding closed-form steps expressed through increments of $M_t^{1/2}$. The projector-based structure further lets any scalar diffusion solver be lifted to the anisotropic setting; instantiating this with DPM-Solver-v3 yields the best FID at very low NFE (Appendix~\ref{s:dpmv3_composition}).
\end{enumerate}

\paragraph{(2) An efficient schedule-gradient estimator for matrix-valued trajectories.}
Optimizing $M_t(\theta)$ is challenging because changing $\theta$ changes the entire forward family $p_t(\cdot;\theta)$, and thus requires $\partial_\theta \nabla \log p_t(x;\theta)$, which is not available from a network trained at a single $\theta$.
We derive a plug-in estimator of $\partial_\theta \nabla \log p_t(x;\theta)$ using only higher-order \emph{$x$-directional} derivatives of the network, implementable in \textbf{three backward passes} and independent of $\dim(\theta)$ (Theorem~\ref{t:del_theta_score}).
We further introduce a \emph{flow parameterization} that stabilizes scale across noise levels and reduces estimator variance (Section~\ref{ss:del_theta_flow}).

\paragraph{(3) Empirical validation of anisotropic trajectory learning.}
On CIFAR-10, AFHQv2, FFHQ, and ImageNet-64 (Section~\ref{s:experiment}),
learned anisotropic trajectories match or improve upon the EDM baseline under
the standard evaluation protocol, with the most pronounced gains in the
low-NFE regime. We further benchmark against specialized few-step samplers
(AYS \citep{sabour2024align}, DPM-Solver-v3 \citep{zheng2023dpm}) and
find that our method, despite operating at the \emph{training} rather than
\emph{solver} level, achieves competitive or better FID. The improvements persist across solver budgets and schedule families,
supporting the view that \emph{learning the matrix trajectory improves
score-matching and the resulting generative model}.

\subsection{Related Work}
Structured forward processes extend diffusion beyond
isotropic noise through blurring, linear corruptions,
and non-isotropic Gaussian covariances
\citep{hoogeboom2022blurring,daras2022soft,
rissanen2022generative,voleti2022score,scimeca2025learning}.
Whitened Score Diffusion~\citep{alido2026whitened}
addresses covariance inversion in anisotropic score matching
through a reparameterized learning target.
Learned forward processes include scalar schedules in
VDM~\citep{kingma2021variational} and per-pixel,
input-adaptive schedules in MuLAN~\citep{sahoo2024diffusion}.
Our focus is joint optimization of a matrix covariance
trajectory and the score network using a trajectory-level
objective, with structured subspace parameterizations
for efficient training and sampling.
Sampler-side methods such as DPM-Solver and AYS
\citep{lu2022dpm,zheng2023dpm,sabour2024align}
provide complementary choices for numerical integration.
Appendix~\ref{ss:related_work} discusses these connections
in more detail.

\section{Preliminaries}
\label{sec:prelim}

\subsection{Anisotropic Diffusion: Process, Score, and Parameterizations}
\label{sec:prelim:process}

Recall the anisotropic diffusion process in \eqref{e:anisotropic_diffusion}. Let $M_t(\theta)$ denote the noise covariance at time $t$, parameterized by $\theta$.  $p_t$ as defined in \eqref{e:anisotropic_diffusion} has score given by
\begin{equation}
\begin{aligned}
    \nabla \log p_t(x;\theta) = M_t^{-1}(\theta) \Ep{x_0 | x_t = x}{x_0 - x_t},
    \label{e:anisotropic_score}
\end{aligned}
\end{equation}
where $(x_0, x_t)$ are defined by the joint distribution $x_0 \sim p_0$ and $x_t = x_0 + \N(0, M_t)$. We provide a short proof in Lemma \ref{l:anisotropic_score} in Appendix \ref{s:theory_foundations}. In case of time-uniform isotropic diffusion (i.e. standard Brownian Motion), $M_t(\theta) = t I$, and the formula in \eqref{e:anisotropic_score} reduces to the standard score expression.

We parameterize a neural network $\net(x,t,\phi)$ to approximate the score, we also define $\flow$, a transformation of $\net$ whose norm is approximately time-invariant:
\begin{align}
&\net(x,t,\phi) \approx \nabla \log p_t(x;\theta)
\label{e:net}\\
&\flow(x,t,\phi) := M_t^{1/2}\net(x,t,\phi)
\label{e:flow-1}
\end{align}

\paragraph{Continuous and discrete reverse ODE for anisotropic denoising.} Recall the \emph{forward ODE dynamics} as defined in \eqref{e:anisotropic_diffusion}: $d x_t=-\tfrac12\partial_t M_t\nabla\log p_t(x_t)dt$. Given $\net$ in \eqref{e:net}, the learned denoising ODE is given by the time-reversal of

\begin{equation}
\begin{aligned}
    & d \bar{x}_t = - \frac{1}{2} \del_t M_t(\theta) \net(\bar{x}_t, t, \phi) dt.
    \label{e:net_forward_ode}
\end{aligned}
\end{equation}

\paragraph{Variance-preserving anisotropic diffusion.} It is often more useful to consider the \emph{variance preserving} anisotropic diffusion, which is simply a time-dependent linear-transformation of \eqref{e:anisotropic_diffusion}. Define
\begin{equation}
\begin{aligned}
    x_t^{VP} := \lrp{I + M_t(\theta)}^{-1/2} x_t.
    \label{e:anisotropic_variance_preserving}
\end{aligned}
\end{equation}
The choice of $I$ above is based on the assumption that $Cov(x_0) \approx I$, which implies that $Cov(x^{VP}_t)\approx I$ for all $t$. Combining \eqref{e:anisotropic_variance_preserving} with \eqref{e:anisotropic_diffusion}, and using that $M_t(\theta)$ commutes with $\partial_t M_t(\theta)$ (as for all families in Appendix~\ref{s:implementation_details}), gives the following dynamics:
\begin{equation}
\begin{aligned}
    &d x_t^{VP} = v(x_t,t; \theta) dt, \qquad  \text{where}\\
    &v(x,t;\theta) := -\frac{1}{2} (I+M_t(\theta))^{-1/2}\del_t M_t(\theta) \blue{\nabla \log p_t(x;\theta)} - \frac{1}{2} (I + M_t(\theta))^{-3/2} \del_t M_t(\theta) x,
    \label{e:true_velocity}
\end{aligned}
\end{equation}
where $v(x,t;\theta)$ is the \emph{ideal} velocity field. Analogously, we define the \emph{learned} velocity field
\begin{equation}
\begin{aligned}
    &d \bar{x}_t^{VP} = \bar{v}(x_t,t; \theta) dt, \qquad  \text{where}\\
    & \bar{v}(x,t;\theta,\phi) := -\frac{1}{2} (I+M_t(\theta))^{-1/2}\del_t M_t(\theta) \red{\net(x,t,\phi)} - \frac{1}{2} (I + M_t(\theta))^{-3/2} \del_t M_t(\theta) x,
    \label{e:learned_velocity}
\end{aligned}
\end{equation}

\section{Trajectory-Level Score Matching Loss}
\label{sec:pathwise-objective}

This section introduces our central training objective \eqref{e:L_simplified}, which enables us to \textbf{jointly} learn a
\textbf{matrix-valued} noise path $\{M_t(\theta)\}_{t\in[0,T]}$ together with a \textbf{score network}
$\net(x,t,\phi)$. At a high level, the loss measures a
\textbf{trajectory-weighted score error} induced by the schedule $M_t(\theta)$, and it is designed so
that (for any fixed $\theta$) the optimal network recovers the true score (Lemma~\ref{l:mt_loss_score_matching}).

\subsection{Defining the Trajectory Loss}
\label{sec:pathwise-drift}

\paragraph{Loss definition.}
We define our trajectory-level training objective as
\begin{alignat*}{2}
    L(\theta,\phi)
    &:= \Ep{t\sim \mathrm{Unif}[0,T],\, x_0\sim p_0,\, \epsilon\sim \N(0,I)}{
    \lrn{
    W_t(\theta)\,
    \lrp{
        M_t(\theta)^{1/2}\net\lrp{x_0 + M_t(\theta)^{1/2}\epsilon,\ t,\ \phi}
        + \epsilon
    }
    }_2^2
    }\\
    &\text{where}\qquad
    W_t(\theta) := \lrp{I+M_t(\theta)}^{-1/2}(\del_t M_t(\theta)) M_t(\theta)^{-1/2}.
    \numberthis \label{e:L_simplified}
\end{alignat*}
In words: we draw a noise level $t$, perturb data by $x_t := x_0 + M_t(\theta)^{1/2}\epsilon$, evaluate
the score network at $(x_t,t)$, and penalize its deviation from the \emph{noise target}
$-\epsilon$, under a \emph{matrix-valued} weight $W_t(\theta)$ determined by the
trajectory $M_t(\theta)$. Since $x_t = x_0 + M_t(\theta)^{1/2}\epsilon$, we have
$M_t(\theta)^{-1/2}\epsilon = -M_t(\theta)^{-1}(x_0-x_t)$, so the network term inside
\eqref{e:L_simplified} is trained to match the conditional denoising direction. The operator
$W_t(\theta)$ serves as a \emph{matrix-valued} weighting schedule that can emphasize different subspaces
at different times, and it is exactly this dependence on $M_t(\theta)$ that allows us to \emph{optimize
the trajectory $\theta$}.

\paragraph{Equivalence of loss to velocity mismatch norm.}
Recall from Section~\ref{sec:prelim} that under the variance-preserving (VP) reparameterization
\eqref{e:anisotropic_variance_preserving}, the \emph{ideal} and \emph{learned} VP velocity fields are
$v(\cdot,t;\theta)$ and $\bar v(\cdot,t;\theta,\phi)$ as defined in \eqref{e:true_velocity} and
\eqref{e:learned_velocity}. A natural trajectory-level principle is to choose $(\theta,\phi)$ so that
the learned denoising dynamics match the ideal ones along the reverse trajectory. The continuous-time form of this principle is the path integral
\begin{equation}
\begin{aligned}
    \int_0^T \Ep{x_0\sim p_0,\, \epsilon\sim \N(0,I)}{
    \lrn{
        \bar v(x_t,t;\theta,\phi) - \tilde v(x_t,x_0,t;\theta)
    }_2^2
    }\,dt,
    \qquad x_t = x_0 + M_t(\theta)^{1/2}\epsilon,
\end{aligned}
\label{e:mt_loss}
\end{equation}
where $\tilde v(\cdot,x_0,t;\theta)$ is obtained from the ideal velocity $v(\cdot,t;\theta)$ by
replacing the intractable score $\nabla\log p_t(\cdot;\theta)$ with the single-sample proxy
$M_t(\theta)^{-1}(x_0-x_t)$ (see the score identity \eqref{e:anisotropic_score}):
\begin{equation}
\begin{aligned}
\tilde v(x_t,x_0,t;\theta)
:=&-\frac12 (I+M_t(\theta))^{-1/2}\del_t M_t(\theta)\,M_t(\theta)^{-1}(x_0-x_t)-\frac12 (I+M_t(\theta))^{-3/2}\del_t M_t(\theta)\,x_t.
\end{aligned}
\label{e:proxy-velocity}
\end{equation}
Expanding $\bar v$ and $\tilde v$, the common contraction term
$-\frac12 (I+M_t)^{-3/2}\del_t M_t\,x_t$ cancels, and \eqref{e:mt_loss} reduces to the compact
matrix-weighted score-matching form \eqref{e:L_simplified} (up to an overall constant factor and the
standard Monte Carlo approximation of the time integral by sampling $t$).

\subsection{Key Properties and Interpretation}

\paragraph{Exact score at optimality.}
As standard in the literature \citep{song2019generative,song2021score_sde}, we show that for any fixed
schedule $M_t(\theta)$ satisfying \eqref{e:Mt-psd}, the optimal network for $L(\theta,\phi)$ recovers
the true score in the infinite-data/infinite-capacity limit:
\begin{lemma}[Exact score at optimality]
\label{l:mt_loss_score_matching}
Fix any schedule $M_t(\theta)$ satisfying \eqref{e:Mt-psd}. In the limit of infinite data and model
capacity, the minimizer $\phi^*(\theta) = \arg\min_{\phi} L(\theta, \phi)$ satisfies
\begin{equation}
  \net(x,t,\phi^*(\theta)) = \nabla_x \log p_t(x;\theta).
  \label{e:score-optimality}
\end{equation}
\end{lemma}
We defer the proof of Lemma~\ref{l:mt_loss_score_matching} to Appendix~\ref{ss:l:mt_loss_score_matching}.

\paragraph{Connection to pathwise change-of-measure (Girsanov).}
The objective \eqref{e:mt_loss} is a pathwise integral of a velocity mismatch, aligning with the
classical change-of-measure control suggested by Girsanov's theorem. Concretely, for two processes
$dx_t = v(x_t,t)\,dt + dB_t$ and $d\bar x_t = \bar v(\bar x_t,t)\,dt + dB_t$, the path-space KL is
bounded by a constant times $\int_0^T \E{\|\bar v(x_t,t)-v(x_t,t)\|_2^2}\,dt$ (under standard
regularity, e.g. Novikov's condition). Related bounds have been used to control reverse
\emph{SDE} discretization error \citep{chen2022sampling}. Our loss can be viewed as a tractable
surrogate in this spirit: it replaces the intractable score-dependent $v$ with the proxy velocity
$\tilde v$ in \eqref{e:proxy-velocity}.

% \paragraph{Pathwise change-of-measure interpretation.}
% Equation~\eqref{e:mt_loss} is a path integral of velocity mismatch which, via Girsanov's theorem, bounds the path-space KL between ideal and learned reverse-time dynamics~\citep{chen2022sampling}; \eqref{e:L_simplified} is a tractable surrogate in this spirit.

\paragraph{Consistency check: isotropic trajectory loss reduces to weighted score matching.}
When $M_t(\theta)=g(t;\theta)I$ for a scalar function $g$, the weighting operator $W_t(\theta)$ becomes
the scalar function $W_t(\theta)
=
\frac{\partial_t g(t;\theta)}
{\sqrt{g(t;\theta)}\sqrt{1+g(t;\theta)}}\,I.$ In
Lemma~\ref{l:weighing_is_geodesic} (Appendix~\ref{s:theory_foundations}), we formally prove the
equivalence to a standard weighted score-matching objective. Thus, in the isotropic case, choosing
$g(t;\theta)$ is essentially choosing a weighting over noise levels. Our main focus, however, is that
\eqref{e:L_simplified} enables \textbf{optimization over matrix-valued trajectories}.

\section{Optimization of the Trajectory Loss}
\label{sec:optimize}

This section explains how we \emph{learn} the matrix-valued trajectory $M_t(\theta)$.
We first emphasize the subtle but crucial relationship between $\theta$ and the learned network: the parametric model
$\net(x,t,\phi)$ has \emph{no explicit} dependence on $\theta$, but the \emph{optimal} network for a
given schedule \emph{does} depend on $\theta$, because it approximates the score of the
schedule-dependent marginal $p_t(\cdot;\theta)=p_0*\N(0,M_t(\theta))$. We therefore present schedule learning as the following bilevel problem:

Let $L(\theta,\phi)$ be as defined in \eqref{e:L_simplified}. Define
\begin{equation}
\begin{aligned}
\phi^*(\theta) &\in \arg\min_{\phi} L(\theta,\phi), \qquad \Leftrightarrow \qquad H(\theta) &:= L(\theta,\phi^*(\theta)) = \min_{\phi} L(\theta,\phi).
\end{aligned}
\label{e:bilevel}
\end{equation}
From \eqref{e:bilevel}, the joint minimization over $(\theta,\phi)$ is equivalent to
\begin{align*}
    \min_{\theta,\phi} L(\theta,\phi) = \min_{\theta} H(\theta).
\end{align*}

Lemma~\ref{l:mt_loss_score_matching} implies that in the infinite-data limit, $\net(x,t,\phi^*(\theta))=\nabla_x\log p_t(x;\theta)$ for any $\theta$.

\paragraph{The challenge of computing $\partial_\theta H(\theta)$:}
When optimizing $\theta$ in \eqref{e:bilevel}, we need $\partial_\theta H(\theta)$. Because $H(\theta)$ evaluates the loss at the \emph{$\theta$-dependent} optimum $\phi^*(\theta)$, differentiating $H$ involves computing the derivative of the optimum score field wrt $\theta$, i.e. $\del_\theta \net(x,t,\phi^*(\theta))\approx \del_\theta \nabla \log p_t(x,\theta)$. Even though $\net(x,t,\phi)$ does not depend on $\theta$, an implicit dependency on $\theta$ is introduced via $\phi^*(\theta)$.

The difficulty arises because at any point during training, \textbf{we only observe the current value of $\phi$.} There is no simple way to directly compute $\del_\theta \net(x,t,\phi^*(\theta))$ or even $\partial_\theta \phi^*(\theta)$. However, there is remarkable structure in $p_t(x,\theta)= p_0 * \mathcal{N}(0, M_t(\theta))$. By exploiting this fact, we show that $\del_\theta \nabla \log p_t(x,\theta)$ can be written using \textbf{only derivatives wrt $x$, without explicit derivatives wrt $\theta$.} In turn, this enables us to approximate $\del_\theta \net(x,t,\phi^*(\theta))$ using only $x$-derivatives of the network.

\subsection{A Plug-in Approximation for $\partial_\theta \nabla \log p_t$ via Stochastic Calculus}
\label{ss:optimization_theorem}
Below we provide an identity for expressing the $\theta$-derivative of the true score $\partial_{\theta} \nabla \log p_t(x,\theta)$ using terms that only depend on entries of $\nabla^2 \log p_t(x,\theta)$.

\begin{theorem}
\label{t:del_theta_score}
Let $\theta\in\Re^c$ and let $j\in\{1,\dots,c\}$ be any coordinate. Then for all $x\in\Re^d$,
\begin{equation}
\begin{aligned}
\partial_{\theta_j}\nabla \log p_t(x;\theta)
=&\ \frac{1}{2}\sum_{i=1}^d
\partial_r\partial_s\,\nabla \log p_t\!\Big(x+r e_i+s\,\partial_{\theta_j}M_t(\theta)\,e_i;\theta\Big)\Big|_{r=s=0}\\
&\ +\ \partial_s\,\nabla \log p_t\!\Big(x+s\,\partial_{\theta_j}M_t(\theta)\,\nabla \log p_t(x;\theta);\theta\Big)\Big|_{s=0}.
\end{aligned}
\label{e:del_theta_score}
\end{equation}
Let $\phi^*(\theta)$ be defined as in \eqref{e:bilevel}. Under
sufficient model capacity and data,
\begin{equation}
\begin{aligned}
\partial_{\theta_j}\net(x,t,\phi^*(\theta))
=&\ \frac{1}{2}\sum_{i=1}^d
\partial_r\partial_s\,\net\Big(x+r e_i+s\,\partial_{\theta_j}M_t(\theta)\,e_i,\ t, \phi^*(\theta)\Big)\Big|_{r=s=0}\\
&\ +\ \partial_s\,\net\!\Big(x+s\,\partial_{\theta_j}M_t(\theta)\,\net(x,t,\phi^*(\theta)),\ t,\phi^*(\theta)\Big)\Big|_{s=0}.
\end{aligned}
\label{e:del_theta_net}
\end{equation}
\end{theorem}

We defer the proof to Appendix~\ref{ss:t:del_theta_score}.

\subsection{Flow Parameterization and Variance Reduction}
\label{ss:del_theta_flow}
A practical concern is that $\|\net(x,t,\phi)\|_2$ typically scales like $\|M_t(\theta)^{-1/2}\|_2$,
which can vary significantly across noise levels and inflate gradient variance.
We therefore work with the normalized vector field
\begin{equation}
\flow(x,t,\phi) := M_t(\theta)^{1/2}\,\net(x,t,\phi),
\label{e:flow-2}
\end{equation}
whose scale is approximately time-invariant. Using Theorem~\ref{t:del_theta_score}, we can also
express $\partial_\theta \flow$ in terms of $x$-directional derivatives:

\begin{corollary}
\label{c:del_theta_flow}
Under the same setup as Theorem~\ref{t:del_theta_score},
\[
\begin{aligned}
\partial_\theta \flow(x,t,\phi^*(\theta))
=&\ \frac{1}{2}\sum_{i=1}^d
\partial_r\partial_s\,\flow\!\Big(x+r e_i+s\,\partial_\theta M_t(\theta)\,e_i,\ t,\phi^*(\theta)\Big)\Big|_{r=s=0}\\
&\ +\ \partial_s\,\flow\!\Big(x+s\,M_t(\theta)^{-1/2}\,\partial_\theta M_t(\theta)\,\flow(x,t,\phi^*(\theta)),\ t,\phi^*(\theta)\Big)\Big|_{s=0}\\
&\ +\ \frac{1}{2}\,M_t(\theta)^{-1}\big(\partial_\theta M_t(\theta)\big)\,\flow(x,t,\phi^*(\theta)).
\end{aligned}
\]
\end{corollary}

We defer the proof of Corollary~\ref{c:del_theta_flow} to Appendix~\ref{ss:c:del_theta_flow}.

\subsection{Final Optimization Formulation for $H(\theta)=\min_\phi L(\theta,\phi)$}
\label{ss:final_opt}

We now clarify how the schedule gradient uses \eqref{e:del_theta_net} (or equivalently
Corollary~\ref{c:del_theta_flow}). We emphasize that the additional $\partial_\theta \net(x,t,\phi^*(\theta))$ term appears when differentiating the \textbf{outer value function} $H(\theta)=\min_\phi L(\theta,\phi)$, not when differentiating $L(\theta,\phi)$ with $\phi$ held fixed.

For notational clarity we treat $\theta$ as a scalar below; the vector case follows by applying the same derivation to each coordinate $\theta_j$. Let $x_0\sim p_0$, $\epsilon\sim\N(0,I)$, sample $t$ as in Section~\ref{sec:pathwise-objective}, and set
$x_t:=x_0+M_t(\theta)^{1/2}\epsilon$. Using the flow form of the loss, we verify that $L(\theta,\phi)$ of \eqref{e:L_simplified} is equivalent to
\begin{alignat*}{2}
L(\theta,\phi)
&= \Ep{x_0,\epsilon,t}{\Big\| {W}_t(\theta)\big(\flow(x_t,t,\phi)+\epsilon\big)\Big\|_2^2}
\numberthis \label{e:L_simplified2}
&\qquad
{W}_t(\theta)
:= (I+M_t(\theta))^{-1/2}\,\partial_t M_t(\theta)\,M_t(\theta)^{-1/2}.
\end{alignat*}

We define the convenient notation $\ell(v,\epsilon,W):=\|W(\theta)(v+\epsilon)\|_2^2$, so that $L(\theta,\phi) = \Ep{x_0,\epsilon,t}{\ell(\flow(x_t,t,\phi),\epsilon, W_t(\theta))}$.

\paragraph{Chain rule for the outer objective.}
Define $\flow^*_\theta(x,t):=\flow(x,t,\phi^*(\theta))$ and recall
$H(\theta)=L(\theta,\phi^*(\theta))$. Differentiating gives
\begin{equation}
\begin{aligned}
\partial_\theta H(\theta)
=
\underbrace{\partial_\theta L(\theta,\phi)\Big|_{\phi=\phi^*(\theta)}}_{\text{explicit dependence through }M_t(\theta)}
\;+\;
\underbrace{\Ep{x_0,\epsilon,t}{\Big\langle
\nabla_{v}\ell(\flow(x_t,t,\phi^*(\theta)),\epsilon,W_t(\theta)),\ \partial_\theta \flow(x_t,t,\phi^*(\theta))
\Big\rangle}}_{\text{implicit dependence through the optimal score/flow field}},
\end{aligned}
\label{e:optimize_L_score}
\end{equation}
The second term on the RHS above uses the
(per-sample) gradient of $\ell(v,\epsilon,t)$ with respect to $v$ given by $\nabla_{\flow}\ell_t = 2\,W_t(\theta)^\top W_t(\theta)\,\big(v+\epsilon\big)$, this term isolates the dependence of $H(\theta)$ on $\phi^*(\theta)$. We estimate $\partial_\theta \flow^*_\theta(x_t,t)$ using
Corollary~\ref{c:del_theta_flow} (equivalently, estimate $\partial_\theta \net^*_\theta$ using
\eqref{e:del_theta_net}). In practice, the expectation in \eqref{e:optimize_L_score} is approximated by minibatch averages over
samples $(x_0^{(j)},\epsilon^{(j)},t^{(j)})$, for $j=1...\text{BatchSize}$.

\section{Denoising Algorithms for Matrix Trajectories}
\label{sec:inference_matrix}

This section describes sampling (denoising) given (i) a learned PSD matrix trajectory $M_t(\theta)$ and
(ii) a trained network $\flow(x,t,\phi)$, where
\[
\flow(x,t,\phi) \approx M_t(\theta)^{1/2}\,\nabla_x \log p_t(x;\theta).
\]
We focus on reverse-time ODE sampling, which is the inference method used throughout our experiments. We remark that our method generalizes easily to the reverse-time SDE setting. Via the projector-based parameterization of $M_t(\theta)$, scalar diffusion solvers (e.g., DPM-Solver-v3) can also be lifted to the anisotropic setting through subspace decomposition (Appendix~\ref{s:dpmv3_composition}).

% \xc{Do you want to add a reference to the appendix where you talk about DPM?}

\subsection{Reverse-time ODE in Flow Form and Initialization}

Using $\net(x,t,\phi)=M_t(\theta)^{-1/2}\flow(x,t,\phi)$, and the learned ODE from \eqref{e:net_forward_ode}, the reverse-time ODE is equivalent to the time-reversal of the following:
\begin{equation}
d\bar x_t
=
-\frac12\,\partial_t M_t(\theta)\,M_t(\theta)^{-1/2}\,\flow(\bar x_t,t,\phi)\,dt,
\label{e:reverse_ode_flow}
\end{equation}

For initialization we sample $\xi\sim\mathcal N(0,I)$ and set
\begin{equation}
\bar x_T := M_T(\theta)^{1/2}\xi,
\label{e:init}
\end{equation}
with $T$ chosen large enough that $p_T$ is close to a centered Gaussian.

\paragraph{Why the solver uses increments of $M_t^{1/2}$.}
In this section, we will assume that the matrices $M_t(\theta)$ share eigenvectors and therefore commute with $\partial_t M_t(\theta)$. This assumption is not necessary, but helps significantly simplify the algebra. In Section \ref{s:implementation_details}, we discuss numerous examples of families of $M_t(\theta)$ which satisfy this assumption. Under this commutation,
\(
\frac12\,\partial_t M_t(\theta)\,M_t(\theta)^{-1/2}
=
\partial_t\!\big(M_t(\theta)^{1/2}\big),
\)
so \eqref{e:reverse_ode_flow} becomes
\begin{equation}
d\bar x_t = -\partial_t\!\big(M_t(\theta)^{1/2}\big)\,\flow(\bar x_t,t,\phi)\,dt.
\label{e:reverse_ode_sqrt}
\end{equation}
This reveals that the natural step size for anisotropic denoising is the \emph{matrix increment}
$\Delta M_t^{1/2}$, directly generalizing the scalar EDM viewpoint where $M_t=\sigma(t)^2I$ and
$\Delta M_k^{1/2}=(\sigma_{k-1}-\sigma_k)I$.

\subsection{Reverse Euler Discretization}
\label{ss:euler_anisotropic}

Let $K$ be the number of solver steps and choose a discretization grid
\begin{equation}
0=t_0 < t_1 < \cdots < t_K = T.
\label{e:grid}
\end{equation}
Reverse-time integration proceeds from $t_K$ down to $t_0$. For convenience, define
\begin{equation}
U_k := M_{t_k}(\theta)^{1/2},\qquad
\Delta U_k := U_{k}-U_{k-1} \quad (k=1,\dots,K).
\label{e:Uk}
\end{equation}

A first-order explicit Euler discretization of the \emph{time-reversal} of \eqref{e:reverse_ode_sqrt} yields
\begin{equation}
\bar x_{t_{k-1}}
=
\bar x_{t_k} + \Delta U_k\,\flow(\bar x_{t_k},t_k,\phi),
\qquad k=1,\dots,K.
\label{e:reverse_euler}
\end{equation}
We present the above Euler-Discretization scheme \eqref{e:reverse_euler} mainly as a baseline derivation; in our experiments, second-order Heun
consistently dominates Euler in FID per NFE.

\subsection{Heun's Second-order Method for Matrix Trajectories}
\label{ss:heun_anisotropic}
We now present the matrix-valued generalization of Heun's second-order algorithm. We inherit the notation from Section \ref{ss:euler_anisotropic} above. Fix a secondary evaluation time $\hat t_k\in[t_{k-1},t_k]$ and denote $\hat U_k := M_{\hat t_k}(\theta)^{1/2}$.
Let
\(
u_k := -\flow(\bar x_{t_k},t_k,\phi).
\)
We form a predictor at $\hat t_k$ by
\begin{equation}
\hat x_{\hat t_k}
=
\bar x_{t_k} + (\hat U_k-U_k)\,u_k,
\qquad
\hat u_k := -\flow(\hat x_{\hat t_k},\hat t_k,\phi).
\label{e:heun_predictor}
\end{equation}

At any time $s\in[t_{k-1},t_k]$, we approximate the intermediate flow field at $s$ by the affine model
\(
\tilde u(s)
:=
u_k + (U(s)-U_k)\,(\hat U_k-U_k)^{-1}\,(\hat u_k-u_k),\)
with \(U(s):=M_{s}(\theta)^{1/2}.\) Substituting this approximation into \eqref{e:reverse_ode_sqrt} and integrating yields the closed-form
second-order update
\begin{equation}
\bar x_{t_{k-1}}
=
\bar x_{t_k}
-
\Delta U_k\,u_k
+
\frac12\,(\Delta U_k)^2
(\hat U_k-U_k)^{-1}
(\hat u_k-u_k).
\label{e:heun_matrix}
\end{equation}
We defer the proof to Lemma \ref{l:heun_simplified_computation} in the appendix. When $\hat t_k=t_{k-1}$ (endpoint choice), we have $\hat U_k=U_{k-1}$ and \eqref{e:heun_matrix} reduces to the
familiar Heun form
\(
\bar x_{t_{k-1}}
=
\bar x_{t_k}
-
\Delta U_k\,\frac{u_k+\hat u_k}{2},
\)
which matches the standard scalar EDM/Heun update when $M_t=\sigma(t)^2I$.
In experiments we consider two choices of $\hat{t}_k$, i.e. $\hat t_k=t_{k-1}$ (endpoint) and $\hat t_k=(t_{k-1}+t_k)/2$ (midpoint).

The solver grid
$\{t_k\}_{k=0}^K$ is treated as an orthogonal design choice. In our experiments, the grid is uniform over $[0,T]$, i.e., $t_k=kT/K$.

\paragraph{Efficient implementation under structured parameterizations.}
All updates require only (i) matrix--vector products with $M_t(\theta)^{1/2}$ and (ii) product of
$(\hat U_k-U_k)^{-1}$ with vectors. Under the projector-based families in Section~\ref{s:implementation_details},
these operations reduce to inexpensive subspace-wise scalings (and subspace-wise pseudoinverses), avoiding explicit
$d\times d$ matrix square roots or dense linear algebra.

\section{Experiments}
\label{s:experiment}

We evaluate learned anisotropic diffusion schedules on four standard image generation benchmarks:
CIFAR-10 ($32\times 32$)~\citep{krizhevsky2009learning},
AFHQv2 ($64\times 64$)~\citep{choi2020stargan},
FFHQ ($64\times 64$)~\citep{karras2019style},
and ImageNet-64 ($64\times 64$)~\citep{deng2009imagenet}.

\paragraph{Baselines.} We compare against two groups:
\begin{itemize}[leftmargin=*,nosep]
    \item \textbf{(a) EDM}~\citep{karras2022elucidating}: matched-architecture isotropic-diffusion reference with Heun sampling.
    \item \textbf{(b) Few-step solver-side baselines}: AYS~\citep{sabour2024align} (learned time discretization) and DPM-Solver-v3~\citep{zheng2023dpm}, the latest DPM-Solver variant (after DPM-Solver and DPM-Solver++) and a state-of-the-art training-free solver for few-NFE.
\end{itemize}
Our framework acts at the \emph{training} level and is complementary to (b); we therefore also report \emph{Ours\,+\,DPM-v3}, where the orthogonal-projector structure of $M_t(\theta)$ decouples dynamics into per-subspace scalar diffusions to which DPM-Solver-v3 applies directly (Appendix~\ref{s:dpmv3_composition}).

\paragraph{Training protocol.}
\label{s:training_pipeline}
% Our models are initialized from the corresponding pretrained EDM networks and then fine-tuned while learning the schedule parameters $\theta$ (and continuing to update network parameters $\phi$) as in \eqref{e:L_simplified2}. Across all datasets, fine-tuning consumes the equivalent of $1.2$M image passes. Full optimization, distributed-training, and EMA details are given in Appendix~\ref{ss:training_setup}. To rule out the possibility that gains come merely from additional fine-tuning compute, we also fine-tune EDM for the same 1.2M image passes and observe no comparable improvement; see Appendix \ref{ss:edm_same_compute}.

Our models are initialized from the corresponding pretrained EDM networks and then fine-tuned while learning the schedule parameters $\theta$ (and continuing to update network parameters $\phi$) as in \eqref{e:L_simplified2}. Across all datasets, fine-tuning consumes the equivalent of $1.2$M image passes. Full optimization, distributed-training, and EMA details are given in Appendix~\ref{ss:training_setup}. To isolate the effect of schedule learning from raw fine-tuning compute, we further fine-tune EDM on ImageNet-64 (the largest benchmark) under both EDM's own recipe and our hyperparameters, and find FID within run-to-run variance of the released checkpoint (Appendix~\ref{ss:edm_same_compute}).

\paragraph{Evaluation protocol.}
We generate $50$k samples per configuration and compute Fr\'echet Inception
Distance (FID$\downarrow$). For both EDM and our method we repeat sampling
three times with independent random seeds and report the \emph{minimum} FID,
matching the EDM evaluation protocol~\citep{karras2022elucidating}. Solver
budgets are reported in number of function evaluations (NFE). EDM's Heun sampler with $N$ steps uses NFE $=2N-1$, since its second-order
correction is skipped on the final step to $\sigma=0$~\citep{karras2022elucidating}.
At each paired budget, our midpoint matrix Heun sampler
(Section~\ref{sec:inference_matrix}) uses NFE $=2N-2$, one fewer than EDM
(see Appendix~\ref{ss:eval_protocol} for the breakdown), so our reported FID is
conservative against EDM. Comparisons
against AYS and DPM-Solver-v3 in Table~\ref{tab:fid-extra} use the
corresponding solver-side methods.

\paragraph{Schedule variants.}
For each dataset we train and evaluate a subset of the following instantiations of $M_t(\theta) = \sum_j g_j(t;\theta)P_j$, which span a range of expressivity: global isotropic schedules $g$, class-conditional isotropic schedules ($\{g_y\}_{y=1}^C$), two-subspace anisotropic schedules over a fixed DCT basis $(g^{\mathrm{DCT}}_1,g^{\mathrm{DCT}}_2)$, two-subspace anisotropic schedules over class-conditional PCA bases $(g^{\mathrm{PCA}}_1,g^{\mathrm{PCA}}_2)$, and the class-conditional analogues $\{(g^{\mathrm{DCT}}_{1,y},g^{\mathrm{DCT}}_{2,y})\}_{y=1}^C$ and $\{(g^{\mathrm{PCA}}_{1,y},g^{\mathrm{PCA}}_{2,y})\}_{y=1}^C$. 
Class conditioning can enter through the subspace basis,
the scalar schedules, or both.
In particular, $(g^{\mathrm{PCA}}_1,g^{\mathrm{PCA}}_2)$
uses class-conditional PCA bases with schedules shared
across classes, whereas
$\{(g^{\mathrm{PCA}}_{1,y},g^{\mathrm{PCA}}_{2,y})\}_{y=1}^C$
makes both the bases and schedules class-dependent.
Detailed definitions appear in Appendix~\ref{ss:schedule_families} and \ref{ss:exp_variants}.

\begin{table*}[ht]
\centering
\small
\setlength{\tabcolsep}{4pt}
\begin{tabular}{c|cccccccc}
\toprule

\multicolumn{9}{c}{\textbf{CIFAR-10}} \\
\midrule
NFE (EDM/Ours) & 9/8 & 11/10 & 13/12 & 15/14 & 17/16 & 35/34 & 59/58 & 79/78 \\
\midrule
EDM & 35.52 & 14.37 & 6.694 & 4.231 & 3.027 & \textbf{1.829} & 1.868 & 1.890 \\
Ours & \textbf{3.839} & \textbf{2.900} & \textbf{2.538} & \textbf{2.332} & \textbf{2.213} & 1.851 & \textbf{1.803} & \textbf{1.803}  \\
\midrule

\multicolumn{9}{c}{\textbf{AFHQv2}} \\
\midrule
NFE (EDM/Ours) & 9/8 & 11/10 & 13/12 & 15/14 & 19/18 & 39/38 & 79/78 & 119/118  \\
\midrule
EDM & 27.98 & 13.66 & 7.587 & 4.746 & 2.986 & 2.075 & 2.042 & 2.046  \\
Ours & \textbf{4.524} & \textbf{2.977} & \textbf{2.513} & \textbf{2.401} & \textbf{2.265} & \textbf{2.070} & \textbf{2.015} & \textbf{2.010} \\
\midrule

\multicolumn{9}{c}{\textbf{FFHQ}} \\
\midrule
NFE (EDM/Ours) & 9/8 & 11/10 & 13/12 & 15/14 & 19/18 & 39/38 & 79/78 & 119/118  \\
\midrule
EDM & 57.14 & 29.39 & 15.81 & 9.769 & 5.169 & 2.575 & 2.391 & 2.374  \\
Ours & \textbf{4.967} & \textbf{3.776} & \textbf{3.304} & \textbf{3.011} & \textbf{2.711} & \textbf{2.350} & \textbf{2.297} & \textbf{2.318} \\
\midrule

\multicolumn{9}{c}{\textbf{ImageNet-64}} \\
\midrule
NFE (EDM/Ours) & 9/8 & 11/10 & 13/12 & 15/14 & 19/18 & 39/38 & 79/78 & 119/118  \\
\midrule
EDM & 35.32 & 15.19 & 8.160 & 5.458 & 3.586 & 2.428 & 2.294 & 2.276 \\
Ours & \textbf{8.366} & \textbf{5.823} & \textbf{4.626} & \textbf{3.992} & \textbf{3.282} & \textbf{2.378} & \textbf{2.259} & \textbf{2.239}  \\
\bottomrule
\end{tabular}

\caption{FID $\downarrow$ vs.\ NFE across datasets (50k samples); NFE columns are paired (EDM\,/\,Ours), with Ours using one fewer evaluation. ``Ours'' reports the class-conditional PCA-basis on CIFAR-10, the class-conditional PCA-basis with conditional schedule on ImageNet-64, and the DCT-basis on AFHQv2 and FFHQ (full per-variant breakdown with standard deviations in Table~\ref{tab:fid-all-error-bar}, Appendix~\ref{s:per_variant_results}). Each entry is the minimum FID across 3 independent generations of 50k images.}

\label{tab:fid-all}
\end{table*}

\begin{table*}[ht]
\centering
\small

% ---------- Left: AYS comparison ----------
\begin{minipage}[t]{0.45\textwidth}
\centering
\setlength{\tabcolsep}{6pt}
\begin{tabular}{l|cccc}
\toprule
\multicolumn{5}{c}{\textbf{FFHQ}} \\
\midrule
NFE  & 10 & 20 & 30 & 50 \\
\midrule
AYS  & 5.430 & 3.290 & 2.870 & 2.620 \\
Ours & \textbf{3.776}& \textbf{2.625} & \textbf{2.407} & \textbf{2.356} \\
\bottomrule
\end{tabular}
\end{minipage}\hfill
% ---------- Right: DPM-Solver-v3 composition ----------
\begin{minipage}[t]{0.54\textwidth}
\centering
\setlength{\tabcolsep}{4.5pt}
\begin{tabular}{l|cccc}
\toprule
\multicolumn{5}{c}{\textbf{CIFAR-10}} \\
\midrule
NFE             & 6              & 8              & 34             & 78 \\
\midrule
DPM-Solver-v3   & 6.314          & 3.845          & 1.913          & 1.914 \\
Ours (Heun)         & 11.03          & 3.618          & \textbf{1.815} & \textbf{1.803} \\
Ours\,+\,DPM-Solver-v3 & \textbf{4.208} & \textbf{3.036} & 2.383          & 2.384 \\
\bottomrule
\end{tabular}
\end{minipage}
\caption{FID $\downarrow$ vs.\ NFE in the few-step regime (50k samples).
\textbf{Left:} our method versus AYS on FFHQ.
\textbf{Right:} composing our method with DPM-Solver-v3 on CIFAR-10; the combination achieves the best FID at low NFE (Appendix~\ref{s:dpmv3_composition}), while our matrix Heun solver alone wins at higher NFE.}
\label{tab:fid-extra}

\end{table*}

\subsection{Discussion of Results}

\label{ss:exp_discussion}
Table~\ref{tab:fid-all} reports FID across solver budgets against the EDM baseline; Table~\ref{tab:fid-extra} compares against solver-side baselines (AYS, DPM-Solver-v3) and the composed sampler. Figures~\ref{fig:cifar10_schedule}--\ref{fig:imagenet_schedule} (Appendix~\ref{ss:learned_schedules}) visualize the learned schedules and anisotropy ratios.

\paragraph{(i) Consistent FID improvement over the EDM baseline.}
Our method consistently outperforms EDM in Table~\ref{tab:fid-all}, with the most pronounced gains at low NFE: at NFE $\approx 9$, FID improves from $35.52\!\to\!3.84$ on CIFAR-10, $27.98\!\to\!4.52$ on AFHQv2, $57.14\!\to\!4.97$ on FFHQ, and $35.32\!\to\!8.37$ on ImageNet-64. At converged-NFE budgets the absolute gap is smaller but still consistent: CIFAR-10 $1.829\!\to\!1.803$, AFHQv2 $2.042\!\to\!2.010$, FFHQ $2.374\!\to\!2.297$, ImageNet-64 $2.276\!\to\!2.239$. The single exception is CIFAR-10 at NFE $35$, where EDM's $1.829$ slightly beats our $1.851$; nevertheless, at all other CIFAR-10 NFE budgets (including the higher-NFE points $59$ and $79$) our method dominates, and our overall best FID across NFEs ($1.803$) is still better than EDM's overall best ($1.829$).

\paragraph{(ii) Matrix-valued anisotropy demonstrates benefits.}
Beyond the isotropic warm-up, matrix-valued schedules are often beneficial. As shown in Table~\ref{tab:fid-all-error-bar} (Appendix~\ref{s:per_variant_results}), on CIFAR-10 the $(g^{\mathrm{PCA}}_1,g^{\mathrm{PCA}}_2)$ schedule improves upon both EDM and the learned isotropic schedule $g$ at every budget except NFE $34$, and on AFHQv2 and FFHQ the DCT anisotropic schedule improves upon EDM at every budget. On ImageNet-64, class-conditional anisotropic variants (over both DCT and PCA bases) are the strongest, with FID approaching $2.24$ at the highest NFE budgets.

\paragraph{(iii) Class-conditional variants matter most on complex conditional data.}
On ImageNet-64, class-conditional schedules improve substantially over shared schedules at low NFE (PCA basis: $8.366$ vs.\ $29.47$ at NFE $8$), and the class-conditional anisotropic variants over DCT and PCA bases perform comparably (within run-to-run variability of each other; see Table~\ref{tab:fid-all-error-bar}). On CIFAR-10, the best FID ($1.803$) comes from $(g^{\mathrm{PCA}}_1, g^{\mathrm{PCA}}_2)$, which uses class-conditional PCA bases with shared schedules, whereas class-conditional isotropic schedules improve over the global isotropic schedule at the lowest budget ($3.691$ vs.\ $4.241$ at NFE $8$).

\paragraph{(iv) Competitive with solver-side baselines, and best when composed.}
Our method outperforms AYS at all NFE budgets on FFHQ (Table~\ref{tab:fid-extra}, left): e.g.\ FID $3.776$ vs.\ AYS's $5.430$ at NFE $=10$ and $2.356$ vs.\ $2.620$ at NFE $=50$. A further practical advantage of our framework is that the projector-based parameterization $M_t(\theta) = \sum_j g_j(t;\theta) P_j$ (Section~\ref{ss:subspace_template}) decouples the diffusion into independent scalar VE processes per subspace, so DPM-Solver-v3 can be applied within each subspace using the corresponding $g_j(t;\theta)$ (Appendix~\ref{s:dpmv3_composition}). The composition (\emph{Ours\,+\,DPM-v3}) substantially outperforms either method alone at very low NFE: FID $4.208$ vs.\ DPM-v3's $6.314$ and our anisotropic Heun's $11.03$ at NFE $=6$, and $3.036$ vs.\ $3.845$ and $3.618$ at NFE $=8$ (Table~\ref{tab:fid-extra}). At higher NFE budgets ($\ge 34$), our matrix Heun solver alone is strongest. The regime split suggests a practical recipe: use \emph{Ours\,+\,DPM-v3} for fast inference ($\text{NFE}\!\le\!10$) and the matrix Heun solver of Section~\ref{sec:inference_matrix} when fidelity matters more than speed.

\section{Conclusion}

We introduced a variational framework for anisotropic diffusion that replaces the scalar noise level with a monotone PSD matrix trajectory $M_t(\theta)$ learned jointly with the score network. Three pieces enable this: a trajectory-level objective with a pathwise change-of-measure interpretation, a plug-in identity for $\partial_\theta \nabla \log p_t$ via $x$-directional derivatives only, and a matrix-valued Heun reverse-time ODE solver. Empirically, learned schedules consistently improve over EDM on CIFAR-10, AFHQv2, FFHQ, and ImageNet-64, with largest gains at low NFE and strongest results from class-conditional schedules or bases. Composing with DPM-Solver-v3 further yields the best FID at very low NFE. 

\begin{ack}
We thank the Duke Compute Cluster (DCC) for providing
the computational resources used in this work. This work received no third-party funding. The authors declare no competing interests.
\end{ack}

%Bibliography
\bibliographystyle{unsrtnat}
\bibliography{references}

@article{karras2022elucidating,
  title={Elucidating the design space of diffusion-based generative models},
  author={Karras, Tero and Aittala, Miika and Aila, Timo and Laine, Samuli},
  journal={Advances in neural information processing systems},
  volume={35},
  pages={26565--26577},
  year={2022}
}

@inproceedings{deng2009imagenet,
  title={Imagenet: A large-scale hierarchical image database},
  author={Deng, Jia and Dong, Wei and Socher, Richard and Li, Li-Jia and Li, Kai and Fei-Fei, Li},
  booktitle={2009 IEEE conference on computer vision and pattern recognition},
  pages={248--255},
  year={2009},
  organization={Ieee}
}

@article{chen2022sampling,
  title={Sampling is as easy as learning the score: theory for diffusion models with minimal data assumptions},
  author={Chen, Sitan and Chewi, Sinho and Li, Jerry and Li, Yuanzhi and Salim, Adil and Zhang, Anru R},
  journal={arXiv preprint arXiv:2209.11215},
  year={2022}
}

@article{krizhevsky2009learning,
  title={Learning multiple layers of features from tiny images},
  author={Krizhevsky, Alex and Hinton, Geoffrey and others},
  year={2009},
  publisher={Toronto, ON, Canada}
}

@inproceedings{choi2020stargan,
  title={Stargan v2: Diverse image synthesis for multiple domains},
  author={Choi, Yunjey and Uh, Youngjung and Yoo, Jaejun and Ha, Jung-Woo},
  booktitle={Proceedings of the IEEE/CVF conference on computer vision and pattern recognition},
  pages={8188--8197},
  year={2020}
}

@inproceedings{karras2019style,
  title={A style-based generator architecture for generative adversarial networks},
  author={Karras, Tero and Laine, Samuli and Aila, Timo},
  booktitle={Proceedings of the IEEE/CVF conference on computer vision and pattern recognition},
  pages={4401--4410},
  year={2019}
}

@article{sabour2024align,
  title={Align your steps: Optimizing sampling schedules in diffusion models},
  author={Sabour, Amirmojtaba and Fidler, Sanja and Kreis, Karsten},
  journal={arXiv preprint arXiv:2404.14507},
  year={2024}
}

@inproceedings{pandey2023complete,
  title={A complete recipe for diffusion generative models},
  author={Pandey, Kushagra and Mandt, Stephan},
  booktitle={Proceedings of the IEEE/CVF International Conference on Computer Vision},
  pages={4261--4272},
  year={2023}
}

@article{singhal2023diffuse,
  title={Where to diffuse, how to diffuse, and how to get back: Automated learning for multivariate diffusions},
  author={Singhal, Raghav and Goldstein, Mark and Ranganath, Rajesh},
  journal={arXiv preprint arXiv:2302.07261},
  year={2023}
}

@inproceedings{du2023flexible,
  title={A flexible diffusion model},
  author={Du, Weitao and Zhang, He and Yang, Tao and Du, Yuanqi},
  booktitle={International Conference on Machine Learning},
  pages={8678--8696},
  year={2023},
  organization={PMLR}
}

@inproceedings{vandersanden2025edge,
  title={Edge-preserving noise for diffusion models},
  author={Vandersanden, Jente and Holl, Sascha and Huang, Xingchang and Singh, Gurprit},
  booktitle={Computer Graphics Forum},
  pages={e70383},
  year={2025},
  organization={Wiley Online Library}
}

@article{sahoo2024diffusion,
  title={Diffusion models with learned adaptive noise},
  author={Sahoo, Subham and Gokaslan, Aaron and De Sa, Christopher M and Kuleshov, Volodymyr},
  journal={Advances in Neural Information Processing Systems},
  volume={37},
  pages={105730--105779},
  year={2024}
}

@inproceedings{huang2024blue,
  title={Blue noise for diffusion models},
  author={Huang, Xingchang and Salaun, Corentin and Vasconcelos, Cristina and Theobalt, Christian and Oztireli, Cengiz and Singh, Gurprit},
  booktitle={ACM SIGGRAPH 2024 conference papers},
  pages={1--11},
  year={2024}
}

@inproceedings{jing2022subspace,
  title={Subspace diffusion generative models},
  author={Jing, Bowen and Corso, Gabriele and Berlinghieri, Renato and Jaakkola, Tommi},
  booktitle={European conference on computer vision},
  pages={274--289},
  year={2022},
  organization={Springer}
}

@article{luo2023videofusion,
  title={Videofusion: Decomposed diffusion models for high-quality video generation},
  author={Luo, Zhengxiong and Chen, Dayou and Zhang, Yingya and Huang, Yan and Wang, Liang and Shen, Yujun and Zhao, Deli and Zhou, Jingren and Tan, Tieniu},
  journal={arXiv preprint arXiv:2303.08320},
  year={2023}
}

@inproceedings{ge2023preserve,
  title={Preserve your own correlation: A noise prior for video diffusion models},
  author={Ge, Songwei and Nah, Seungjun and Liu, Guilin and Poon, Tyler and Tao, Andrew and Catanzaro, Bryan and Jacobs, David and Huang, Jia-Bin and Liu, Ming-Yu and Balaji, Yogesh},
  booktitle={Proceedings of the IEEE/CVF International Conference on Computer Vision},
  pages={22930--22941},
  year={2023}
}

@article{chang2025warped,
  title={How i warped your noise: a temporally-correlated noise prior for diffusion models},
  author={Chang, Pascal and Tang, Jingwei and Gross, Markus and Azevedo, Vinicius C},
  journal={arXiv preprint arXiv:2504.03072},
  year={2025}
}

@article{ho2020denoising,
  title={Denoising diffusion probabilistic models},
  author={Ho, Jonathan and Jain, Ajay and Abbeel, Pieter},
  journal={Advances in neural information processing systems},
  volume={33},
  pages={6840--6851},
  year={2020}
}

@inproceedings{song2021score_sde,
  title={Score-Based Generative Modeling through Stochastic Differential Equations},
  author={Song, Yang and Sohl-Dickstein, Jascha and Kingma, Diederik P. and Kumar, Abhishek and Ermon, Stefano and Poole, Ben},
  booktitle={International Conference on Learning Representations (ICLR)},
  year={2021}
}

@article{ruderman1993statistics,
  title={Statistics of natural images: Scaling in the woods},
  author={Ruderman, Daniel and Bialek, William},
  journal={Advances in neural information processing systems},
  volume={6},
  year={1993}
}

@article{zheng2023dpm,
  title={Dpm-solver-v3: Improved diffusion ode solver with empirical model statistics},
  author={Zheng, Kaiwen and Lu, Cheng and Chen, Jianfei and Zhu, Jun},
  journal={Advances in Neural Information Processing Systems},
  volume={36},
  pages={55502--55542},
  year={2023}
}

@article{lu2022dpm,
  title={Dpm-solver: A fast ode solver for diffusion probabilistic model sampling in around 10 steps},
  author={Lu, Cheng and Zhou, Yuhao and Bao, Fan and Chen, Jianfei and Li, Chongxuan and Zhu, Jun},
  journal={Advances in neural information processing systems},
  volume={35},
  pages={5775--5787},
  year={2022}
}

@inproceedings{rombach2022high,
  title={High-resolution image synthesis with latent diffusion models},
  author={Rombach, Robin and Blattmann, Andreas and Lorenz, Dominik and Esser, Patrick and Ommer, Bj{\"o}rn},
  booktitle={Proceedings of the IEEE/CVF conference on computer vision and pattern recognition},
  pages={10684--10695},
  year={2022}
}

@article{tian2024visual,
  title={Visual autoregressive modeling: Scalable image generation via next-scale prediction},
  author={Tian, Keyu and Jiang, Yi and Yuan, Zehuan and Peng, Bingyue and Wang, Liwei},
  journal={Advances in neural information processing systems},
  volume={37},
  pages={84839--84865},
  year={2024}
}

@article{song2019generative,
  title={Generative modeling by estimating gradients of the data distribution},
  author={Song, Yang and Ermon, Stefano},
  journal={Advances in neural information processing systems},
  volume={32},
  year={2019}
}

@article{kingma2021variational,
  title={Variational diffusion models},
  author={Kingma, Diederik and Salimans, Tim and Poole, Ben and Ho, Jonathan},
  journal={Advances in neural information processing systems},
  volume={34},
  pages={21696--21707},
  year={2021}
}

@article{kingma2023understanding,
  title={Understanding diffusion objectives as the elbo with simple data augmentation},
  author={Kingma, Diederik and Gao, Ruiqi},
  journal={Advances in Neural Information Processing Systems},
  volume={36},
  pages={65484--65516},
  year={2023}
}

@article{alido2026whitened,
  title={Whitened score diffusion: A structured prior for imaging inverse problems},
  author={Alido, Jeffrey and Li, Tongyu and Sun, Yu and Tian, Lei},
  journal={Advances in Neural Information Processing Systems},
  volume={38},
  pages={65168--65195},
  year={2026}
}

@article{hoogeboom2022blurring,
  title={Blurring diffusion models},
  author={Hoogeboom, Emiel and Salimans, Tim},
  journal={arXiv preprint arXiv:2209.05557},
  year={2022}
}

@article{daras2022soft,
  title={Soft diffusion: Score matching for general corruptions},
  author={Daras, Giannis and Delbracio, Mauricio and Talebi, Hossein and Dimakis, Alexandros G and Milanfar, Peyman},
  journal={arXiv preprint arXiv:2209.05442},
  year={2022}
}

@article{rissanen2022generative,
  title={Generative modelling with inverse heat dissipation},
  author={Rissanen, Severi and Heinonen, Markus and Solin, Arno},
  journal={arXiv preprint arXiv:2206.13397},
  year={2022}
}

@article{voleti2022score,
  title={Score-based denoising diffusion with non-isotropic gaussian noise models},
  author={Voleti, Vikram and Pal, Christopher and Oberman, Adam},
  journal={arXiv preprint arXiv:2210.12254},
  year={2022}
}

@article{scimeca2025learning,
  title={Learning What Matters: Steering Diffusion via Spectrally Anisotropic Forward Noise},
  author={Scimeca, Luca and Jiralerspong, Thomas and Earnshaw, Berton and Hartford, Jason and Bengio, Yoshua},
  journal={arXiv preprint arXiv:2510.09660},
  year={2025}
}

%%%%%%%%%%%%%%%%%%%%%%%%%%%%%%%%%%%%%%%%%%%%%%%%%%%%%%%%%%%%
\clearpage
\appendix

\begin{center}
{\Large\bfseries Appendix}
\end{center}
\vspace{1.5em}

% ---------------------------------------------------------------------
% Appendix Contents (manual itemize - no extra package needed).
%
% If you prefer an auto-generated local TOC, add to preamble:
%     \usepackage{etoc}
% then replace the \section*{...} + itemize block below with:
%     \etocsetnexttocdepth{subsection}
%     \etocsettocstyle{\section*{Appendix Contents}}{}
%     \localtableofcontents
% ---------------------------------------------------------------------
\section*{Table of Contents}
\vspace{-1.5em}
\noindent\rule{\linewidth}{0.1pt}
\vspace{-1.5em}
\begingroup
\hypersetup{hidelinks}%
\renewcommand{\baselinestretch}{1.25}\selectfont
\setlength{\parindent}{0pt}%

\hyperref[ss:related_work]{\textbf{A\hspace{1em}Related Work}}%
\dotfill\hyperref[ss:related_work]{\pageref{ss:related_work}}\par
\vspace{4pt}

\hyperref[s:theory_foundations]{\textbf{B\hspace{1em}Theoretical Foundations}}%
\dotfill\hyperref[s:theory_foundations]{\pageref{s:theory_foundations}}\par
\hspace*{2.2em}\hyperref[ss:anisotropic_score_lemma]{B.1\hspace{0.8em}Anisotropic Score as a Conditional Expectation}%
\dotfill\hyperref[ss:anisotropic_score_lemma]{\pageref{ss:anisotropic_score_lemma}}\par
\hspace*{2.2em}\hyperref[ss:weighing_lemma]{B.2\hspace{0.8em}Equivalence of Isotropic Schedules and Weighted Score-Matching}%
\dotfill\hyperref[ss:weighing_lemma]{\pageref{ss:weighing_lemma}}\par
\hspace*{2.2em}\hyperref[ss:heun_lemma]{B.3\hspace{0.8em}Closed-Form Integral for the Matrix Heun Step}%
\dotfill\hyperref[ss:heun_lemma]{\pageref{ss:heun_lemma}}\par
\hspace*{2.2em}\hyperref[s:proofs]{B.4\hspace{0.8em}Proofs of Main Results}%
\dotfill\hyperref[s:proofs]{\pageref{s:proofs}}\par
\hspace*{4.4em}\hyperref[ss:l:mt_loss_score_matching]{B.4.1\hspace{0.8em}Proof of Lemma~\ref{l:mt_loss_score_matching}}%
\dotfill\hyperref[ss:l:mt_loss_score_matching]{\pageref{ss:l:mt_loss_score_matching}}\par
\hspace*{4.4em}\hyperref[ss:t:del_theta_score]{B.4.2\hspace{0.8em}Proof of Theorem~\ref{t:del_theta_score}}%
\dotfill\hyperref[ss:t:del_theta_score]{\pageref{ss:t:del_theta_score}}\par
\hspace*{4.4em}\hyperref[ss:c:del_theta_flow]{B.4.3\hspace{0.8em}Proof of Corollary~\ref{c:del_theta_flow}}%
\dotfill\hyperref[ss:c:del_theta_flow]{\pageref{ss:c:del_theta_flow}}\par
\vspace{4pt}

\hyperref[s:implementation_details]{\textbf{C\hspace{1em}Practical Choices of Matrix Schedule $M_t(\theta)$}}%
\dotfill\hyperref[s:implementation_details]{\pageref{s:implementation_details}}\par
\hspace*{2.2em}\hyperref[ss:subspace_template]{C.1\hspace{0.8em}A Subspace-Schedule Framework}%
\dotfill\hyperref[ss:subspace_template]{\pageref{ss:subspace_template}}\par
\hspace*{2.2em}\hyperref[ss:schedule_families]{C.2\hspace{0.8em}Schedule Families Used in Experiments}%
\dotfill\hyperref[ss:schedule_families]{\pageref{ss:schedule_families}}\par
\hspace*{2.2em}\hyperref[ss:exp_variants]{C.3\hspace{0.8em}Schedule Variants Evaluated}%
\dotfill\hyperref[ss:exp_variants]{\pageref{ss:exp_variants}}\par
\vspace{4pt}

\hyperref[s:impl]{\textbf{D\hspace{1em}Experiment Details}}%
\dotfill\hyperref[s:impl]{\pageref{s:impl}}\par
\hspace*{2.2em}\hyperref[ss:g_i_impl_appendix]{D.1\hspace{0.8em}Subspace Schedule Parameterization $g_i$}%
\dotfill\hyperref[ss:g_i_impl_appendix]{\pageref{ss:g_i_impl_appendix}}\par
\hspace*{2.2em}\hyperref[ss:training_setup]{D.2\hspace{0.8em}Training Setup}%
\dotfill\hyperref[ss:training_setup]{\pageref{ss:training_setup}}\par
\hspace*{2.2em}\hyperref[ss:eval_protocol]{D.3\hspace{0.8em}Evaluation Protocol Details}%
\dotfill\hyperref[ss:eval_protocol]{\pageref{ss:eval_protocol}}\par
\hspace*{2.2em}\hyperref[s:dpmv3_composition]{D.4\hspace{0.8em}Composing Anisotropic Schedules with DPM-Solver-v3}%
\dotfill\hyperref[s:dpmv3_composition]{\pageref{s:dpmv3_composition}}\par
\vspace{4pt}

\hyperref[s:additional_exp]{\textbf{E\hspace{1em}Additional Experimental Analysis}}%
\dotfill\hyperref[s:additional_exp]{\pageref{s:additional_exp}}\par
\hspace*{2.2em}\hyperref[s:per_variant_results]{E.1\hspace{0.8em}Per-Variant FID Breakdown}%
\dotfill\hyperref[s:per_variant_results]{\pageref{s:per_variant_results}}\par
\hspace*{2.2em}\hyperref[ss:learned_schedules]{E.2\hspace{0.8em}Learned Schedules across Datasets}%
\dotfill\hyperref[ss:learned_schedules]{\pageref{ss:learned_schedules}}\par
\hspace*{2.2em}\hyperref[ss:speed_analysis]{E.3\hspace{0.8em}Training and Inference Speed}%
\dotfill\hyperref[ss:speed_analysis]{\pageref{ss:speed_analysis}}\par
\hspace*{2.2em}\hyperref[ss:edm_same_compute]{E.4\hspace{0.8em}Fine-Tuning Controls with Matched Image Budgets}%
\dotfill\hyperref[ss:edm_same_compute]{\pageref{ss:edm_same_compute}}\par
\hspace*{2.2em}\hyperref[s:bases]{E.5\hspace{0.8em}Subspace Bases}%
\dotfill\hyperref[s:bases]{\pageref{s:bases}}\par
\hspace*{4.4em}\hyperref[ss:2d_dct_appendix]{E.5.1\hspace{0.8em}2D-DCT bases}%
\dotfill\hyperref[ss:2d_dct_appendix]{\pageref{ss:2d_dct_appendix}}\par
\hspace*{4.4em}\hyperref[ss:pca_geometry]{E.5.2\hspace{0.8em}PCA bases}%
\dotfill\hyperref[ss:pca_geometry]{\pageref{ss:pca_geometry}}\par
\vspace{4pt}

\hyperref[ss:broader_impacts]{\textbf{F\hspace{1em}Broader Impacts}}%
\dotfill\hyperref[ss:broader_impacts]{\pageref{ss:broader_impacts}}\par
 \vspace{-1em}
\endgroup
\vspace{4pt}
\noindent\rule{\linewidth}{0.4pt}
\vspace{1em}
 
% =====================================================================
\section{Related Work}
\label{ss:related_work}
% =====================================================================
\paragraph{Structured noising and anisotropy.}
Prior work introduces non-isotropic structure into diffusion via edge-aware perturbations \citep{vandersanden2025edge}, correlated or blue-noise masks \citep{huang2024blue}, frequency- or subspace-restricted noise \citep{jing2022subspace}, and temporally structured priors for video models \citep{luo2023videofusion,ge2023preserve,chang2025warped}. These approaches typically fix the structure of the forward process by hand (selecting a particular spatial pattern, frequency mask, or temporal correlation) rather than learning the matrix trajectory from data. In contrast, we learn a \emph{global}, monotone PSD \emph{matrix} trajectory $M_t(\theta)$ end-to-end (e.g., DCT- or PCA-parameterized) and optimize it \emph{jointly} with the score network via a trajectory-level loss, allocating noise (and denoising effort) across \emph{directions and subspaces} in a data-driven way.
 
\paragraph{Multivariate diffusions and learned matrix parameters.}
A closer line of work jointly considers matrix-valued forward-process parameters alongside a score model, but each restricts the matrix family for tractability and stops short of optimizing a general $d\times d$ covariance trajectory.
\citet{pandey2023complete} provide a general recipe for forward-process design in augmented phase space; for tractability they restrict the diffusion and drift matrices $D(z),Q(z)$ to constants and instantiate the resulting PSLD model with a small set of scalar hyperparameters $(\Gamma, M, \nu, \beta)$. These matrices are therefore treated as design choices/hyperparameters set by hand rather than a jointly learned matrix trajectory.
\citet{singhal2023diffuse} optimize an ELBO for multivariate diffusions parameterized through low-dimensional auxiliary-variable dynamics: with $K{-}1$ auxiliary variables, the forward matrices are $K\times K$ (with $K\in\{2,3\}$ in their experiments) and shared across data coordinates; their Appendix C further parameterizes $Q_\phi(s), D_\phi(s)$ as scalar time-functions multiplying time-independent matrices, with $D$ taken diagonal in their experiments. The matrix family is thus low-dimensional and coordinate-shared rather than a general $d\times d$ data-space trajectory.
\citet{du2023flexible}, motivated by Riemannian and symplectic geometry, derive a flexible class of multivariate flow-based models, but for image generation explicitly use a simplified linear FP-Diffusion with $R^{-1}(x)$ taken as a constant SPD matrix independent of $x$; the practical image model thus uses a global constant metric rather than a learned data-dependent manifold.
 
We instead learn a general data-space covariance path $M_t(\theta)\in\mathbb{R}^{d\times d}$ end-to-end via a trajectory-level loss. This generality introduces a key technical challenge: changing $\theta$ changes the target score $\nabla\log p_t(x;\theta)$, which is not available from a network trained at a single $\theta$. Estimating $\partial_\theta \nabla \log p_t(x;\theta)$ efficiently is exactly the purpose of Theorem~\ref{t:del_theta_score}.
 
\paragraph{ELBO views and schedule learning.}
VDM \citep{kingma2021variational} and the ELBO reinterpretation of common diffusion losses \citep{kingma2023understanding} relate \emph{scalar} schedule learning to reweighting \emph{noise levels}. MuLAN \citep{sahoo2024diffusion} learns (possibly input-conditional) \emph{per-pixel} noise rates to tighten an ELBO for density estimation. Our formulation differs in that (i) $M_t(\theta)$ is \emph{anisotropic} (reweights \emph{directions}, not just levels), (ii) we derive a practical schedule-gradient estimator that accounts for the schedule-dependent score target, and (iii) we optimize the schedule to improve \emph{sample quality}, yielding consistent FID gains.
 
\paragraph{Few-step samplers and solver design.}
A complementary line of work targets the \emph{sampler} rather than the model, including DPM-Solver and its higher-order extensions \citep{lu2022dpm,zheng2023dpm} and learned step schedules such as Align-Your-Steps (AYS) \citep{sabour2024align}. Our approach acts on the \emph{training objective} and the \emph{forward schedule} $M_t(\theta)$, and is therefore largely orthogonal to these solver-side methods. In Section~\ref{s:experiment}, we show that our learned anisotropic trajectories match or outperform AYS and DPM-Solver-v3 across a range of NFE budgets, and compose with DPM-Solver-v3 (Appendix~\ref{s:dpmv3_composition}) to yield further gains in the very-low-NFE regime.
 
% =====================================================================
\section{Theoretical Foundations}
\label{s:theory_foundations}
% =====================================================================

\subsection{Anisotropic Score as a Conditional Expectation}
\label{ss:anisotropic_score_lemma}

\begin{lemma}\label{l:anisotropic_score}
If $x_t$ evolves as either one of the two processes:
\begin{align*}
    &1.\ (x_t-x_0) \sim \N(0, M_t(\theta))\\
    &2.\ dx_t = -\frac{1}{2} \del_t M_t(\theta) \nabla \log p_t(x_t; \theta) dt,
\end{align*}
where $p_t(x;\theta):= p_0 * \N(0, M_t(\theta))$, then the score above is the conditional expectation
\begin{align*}
    \nabla \log p_t(x;\theta) = M_t^{-1}(\theta) \Ep{x_0 | x_t = x}{x_0 - x_t},
\end{align*}
where $(x_0, x_t)$ are defined by the joint distribution $x_0 \sim p_0$ and $x_t = x_0 + \N(0, M_t)$.
\end{lemma}

\begin{proof}[Proof of Lemma \ref{l:anisotropic_score}]
Fix $t$ and $\theta$, and write $\Sigma := M_t(\theta)$. For notational simplicity we suppress the
dependence on $(t,\theta)$ when clear. The marginal $p_t(\cdot;\theta)$ is the Gaussian convolution
\[
p_t(x;\theta)=\int_{\Re^d} p_0(x_0)\,\varphi_\Sigma(x-x_0)\,dx_0,
\qquad
\varphi_\Sigma(z):=(2\pi)^{-d/2}|\Sigma|^{-1/2}\exp\!\Big(-\tfrac12 z^\top \Sigma^{-1} z\Big).
\]
Assuming $\Sigma\succ 0$, we may
differentiate under the integral sign to obtain
\[
\nabla_x p_t(x;\theta)=\int p_0(x_0)\,\nabla_x \varphi_\Sigma(x-x_0)\,dx_0.
\]
Since $\nabla_x \varphi_\Sigma(x-x_0)=\Sigma^{-1}(x_0-x)\,\varphi_\Sigma(x-x_0)$, we get
\[
\nabla_x p_t(x;\theta)
=
\Sigma^{-1}\int p_0(x_0)\,(x_0-x)\,\varphi_\Sigma(x-x_0)\,dx_0.
\]
Dividing by $p_t(x;\theta)$ yields
\[
\nabla_x \log p_t(x;\theta)
=
\Sigma^{-1}\,
\frac{\int p_0(x_0)\,(x_0-x)\,\varphi_\Sigma(x-x_0)\,dx_0}{\int p_0(x_0)\,\varphi_\Sigma(x-x_0)\,dx_0}
=
\Sigma^{-1}\,\E{x_0-x \mid x_t=x},
\]
where the conditional distribution is $p(x_0\mid x_t=x)\propto p_0(x_0)\varphi_\Sigma(x-x_0)$.
Finally, since $x_t=x$ on the conditioning event, $x_0-x = x_0-x_t$, giving the stated identity.
\end{proof}

\subsection{Equivalence of Isotropic Schedules and Weighted Score-Matching}
\label{ss:weighing_lemma}

\begin{lemma}
    \label{l:weighing_is_geodesic}
    For any $w(t)$, there exists a $g_t(\theta)$ and constant $c$, such that for any  $h(t)$
    \begin{align*}
        \int_0^T \frac{(\del_t g_t(\theta))^2}{1+g_t(\theta)} h(g_t(\theta)) dt = c \int_0^T w(t) h(t) dt.
    \end{align*}
\end{lemma}

\begin{proof}[Proof of Lemma \ref{l:weighing_is_geodesic}]

Let $w:[0,T]\to(0,\infty)$ be measurable and define
$$
\Phi(x):=\int_{0}^{x}\frac{ds}{(1+s)w(s)}<\infty.
$$
Then, to prove this Lemma, it is sufficient to show there exist a $c=\frac{\Phi(T)}{T}>0$ and a strictly increasing continuous function $g:[0,T]\to[0,T]$ with $g(0)=0,\ g(T)=T$ such that

$$
 \frac{g'\bigl( g^{-1}(t)\bigr)}{1+t} = cw(t) \text{ for all }t.
$$

where $c=\Phi(T)/T$.

Define $g$ implicitly by $\Phi\bigl(g(t)\bigr)=ct\text{ for }(0\le t\le T),$ i.e. $g(t)=\Phi^{-1}(ct)$. Differentiating $\Phi(g(t))=c t$ with respect to $t$ yields the separable ODE $g'(t)=(1+g(t))cw(g(t))$, and with $r=g(t)$ this is $\frac{g'(g^{-1}(r))}{1+r}=cw(r)$. Hence, we derive an expression for $g$ involving a constant $c$ and any $w$. Monotonicity of $g(t)$ follows since $w>0$ and $c>0$ imply $g'(t)>0$.

Substituting into $c \int_0^T w(t)h(t)dt$, we get
$$c \int_0^T w(r) h(r)dr = \int_0^T \frac{g'(t)}{1+g(t)} h(g(t))g'(t)dt.$$

\end{proof}

\subsection{Closed-Form Integral for the Matrix Heun Step}
\label{ss:heun_lemma}

\begin{lemma}
    \label{l:heun_simplified_computation}
    Assume $M_t$ commutes with $\partial_t M_t$ for all $t$. Let $\tilde{u}$ be as defined in Section~\ref{sec:inference_matrix}. Then
    \begin{align*}
        & \int_{t_k}^{t_{k-1}} (\del_{s} M_{s}^{1/2}) \tilde{u}(s)\, ds\\
        &\qquad = (M_{t_{k}}^{1/2}-M_{t_{k-1}}^{1/2}) \lrp{\flow(\bar{x}_{t_k},t_k)}
        - \frac{1}{2}(M_{t_{k-1}}^{1/2} - M_{t_k}^{1/2})^2 \big(M_{\hat{t}_{k}}^{1/2} - M_{t_k}^{1/2}\big)^{-1}
         \cdot \lrp{\flow(\hat{x}_{\hat{t}_{k}},\hat{t}_k) - \flow(\bar{x}_{t_k},t_k)}.
    \end{align*}
\end{lemma}
\begin{proof}
    Recall our notation that
    \begin{align*}
        u_k := -\flow(\bar x_{t_k},t_k,\phi), \qquad \hat u_k := -\flow(\hat x_{\hat t_k},\hat t_k,\phi).
    \end{align*}
    \begin{align*}
         & \int_{t_k}^{t_{k-1}} (\del_{s} M_{s}^{1/2}(\theta)) \tilde{u}(s)\, ds\\
         =& \int_{t_k}^{t_{k-1}} (\del_{s} M_{s}^{1/2}(\theta)) \lrp{u_k + (M_{s}^{1/2}(\theta)-M_{t_k}^{1/2}(\theta))\,(M_{\hat{t}_k}^{1/2}(\theta)-M_{t_k}^{1/2}(\theta))^{-1}\,(\hat u_k-u_k)} ds\\
         =& \lrp{M_{t_{k-1}}^{1/2}(\theta)-M_{t_k}^{1/2}(\theta)}u_k + \int_{t_k}^{t_{k-1}} (\del_{s} M_{s}^{1/2}(\theta)) (M_{s}^{1/2}(\theta)-M_{t_k}^{1/2}(\theta))\, ds \ \  (M_{\hat{t}_k}^{1/2}(\theta)-M_{t_k}^{1/2}(\theta))^{-1}(\hat u_k-u_k)\\
         =& \lrp{M_{t_{k-1}}^{1/2}(\theta)-M_{t_k}^{1/2}(\theta)}u_k + \frac{1}{2} (M_{t_{k-1}}^{1/2}(\theta)-M_{t_k}^{1/2}(\theta))^2 (M_{\hat{t}_k}^{1/2}(\theta)-M_{t_k}^{1/2}(\theta))^{-1}(\hat u_k-u_k).
    \end{align*}

\end{proof}

% =====================================================================
\subsection{Proofs of Main Results}
\label{s:proofs}
% =====================================================================

\subsubsection{Proof of Lemma \ref{l:mt_loss_score_matching}}
\label{ss:l:mt_loss_score_matching}
We simplify the term inside the Euclidean norm in \eqref{e:mt_loss}:
\begin{align*}
    & \bar{v}(x_t,t;\theta,\phi) - \tilde{v}(x_t,x_0,t;\theta) \\
    & = \frac12 (I+M_t(\theta))^{-1/2}\del_t M_t(\theta) (\blue{M_t^{-1}(\theta)(x_0-x_t)} -\red{\net(x_t,t,\phi)}).
\end{align*}
Let $M_{\text{all}}$ denote $(I+M_t(\theta))^{-1/2}\del_t M_t(\theta)$. Let $V: \Re^d \to \Re^d$ be an arbitrary vector field.

We can minimize the expectation pointwise at each $x_t$. We can rewrite the expectation
$$\Ep{x_0, \epsilon}{\lrn{M_{\text{all}} (\red{\net(x_t,t,\phi)} - \blue{M_t^{-1}(\theta)(x_0-x_t)}) }_2^2}$$
as
\begin{align*}
& \Ep{x_0, \epsilon}{\lrn{M_{\text{all}} (\red{V(x_t)} - \blue{M_t^{-1}(\theta)(x_0-x_t)}) }_2^2}\\
& =\Ep{x_t}{\Ep{x_0, \epsilon | x_t}{\lrn{M_{\text{all}} (\red{V(x_t)} - \blue{M_t^{-1}(\theta)(x_0-x_t)}) }_2^2}}
\end{align*}
by the law of iterated expectation.
\begin{align*}
    &\arg\min_{v\in \Re^d}
    \mathbb{E}_{x_0,\epsilon\mid x_t}
    \Big[
    (v-M_t^{-1}(\theta)(x_0-x_t))^\top \\
    &\qquad\qquad
    M_{\text{all}}^\top M_{\text{all}}
    (v-M_t^{-1}(\theta)(x_0-x_t))
    \Big] \\
    &=
    M_t^{-1}(\theta)
    \Ep{x_0,\epsilon\mid x_t}{x_0-x_t}
    =
    \nabla\log p_t(x_t;\theta).
\end{align*}

Let $Q=M_{\text{all}}^\top M_{\text{all}}$.
Under \eqref{e:Mt-psd}, $M_{\text{all}}$ is invertible,
so $Q$ is positive definite.
The conditional quadratic loss is therefore uniquely minimized at
$M_t(\theta)^{-1}\mathbb E[x_0-x_t\mid x_t]$,
which equals the score by \eqref{e:anisotropic_score}.

\subsubsection{Proof of Theorem \ref{t:del_theta_score}}
\label{ss:t:del_theta_score}

To simplify notation, we will drop the index $j$ and treat $\theta$ as a scalar. The general proof for $\theta\in\Re^c$ follows by repeating the proof for each $\theta_j$, while holding all other $\theta_k's$ fixed. Recall that $p_t(\cdot;\theta) = p_0 * \N(0, M_t(\theta))$. By the heat equation (applied to $\del_\theta M_t(\theta)$ instead of $\del_t M_t(\theta)$),

\begin{align*}
    &\partial_\theta p_t(x;\theta) \\
    =& \frac{1}{2} \div(p_t(x;\theta)\del_\theta {M_t(\theta)} \nabla \log p_t(x;\theta))\\
    =& \frac{1}{2} p_t(x;\theta) \div\lrp{\del_\theta {M_t(\theta)} \nabla \log p_t(x;\theta)} \\
    &+ \frac{1}{2} p_t(x;\theta)\lin{\nabla \log p_t(x;\theta), \del_\theta {M_t(\theta)} \nabla \log p_t(x;\theta)}.
\end{align*}

Dividing both sides by $p_t(x;\theta)$ gives
\begin{align*}
&\partial_\theta \log p_t(x;\theta) \\
=& \frac{1}{2}\div\lrp{\del_\theta {M_t(\theta)} \nabla \log p_t(x;\theta)} \\
&+ \frac{1}{2}\lin{\nabla \log p_t(x;\theta), \del_\theta {M_t(\theta)} \nabla \log p_t(x;\theta)}\\
=& \frac{1}{2} \sum_i \lin{\del_\theta {M_t(\theta)} e_i,  \nabla^2 \log p_t(x;\theta) e_i}\\
&+ \frac{1}{2} \lin{\nabla \log p_t(x;\theta), \del_\theta {M_t(\theta)} \nabla \log p_t(x;\theta)}\\
=& \frac{1}{2} \sum_i \lin{\del_\theta {M_t(\theta)} e_i,  \del_c \nabla \log p_t(x+ce_i;\theta)} \\
& + \frac{1}{2} \lin{\nabla \log p_t(x;\theta), \del_\theta {M_t(\theta)} \nabla \log p_t(x;\theta)}.
\end{align*}
Taking a derivative wrt $x$ gives
\begin{align*}
&\partial_\theta \nabla \log p_t(x;\theta) \\
=& \frac{1}{2} \sum_i \lin{\del_\theta {M_t(\theta)} e_i,  \del_c \nabla^2 \log p_t(x+ce_i;\theta)} \\
& + \lin{\del_\theta {M_t(\theta)} \nabla \log p_t(x;\theta), \nabla^2 \log p_t(x;\theta)}\\
=& \frac{1}{2} \sum_i \del_r \del_s \nabla \log p_t(x+re_i+s\del_\theta {M_t(\theta)} e_i;\theta)\\
&+ \del_s \nabla \log p_t(x+s\del_\theta {M_t(\theta)} \nabla \log p_t(x;\theta);\theta).
\end{align*}
This concludes the proof of \eqref{e:del_theta_score}. To prove \eqref{e:del_theta_net}, we combine \eqref{e:del_theta_score} with the guarantee that $\net(x,t,\phi^*(\theta)) = \nabla \log p_t(x;\theta)$ from Lemma \ref{l:mt_loss_score_matching}.

\subsubsection{Proof of Corollary \ref{c:del_theta_flow}}
\label{ss:c:del_theta_flow}

Recall that
$$\flow(x;\theta) = M_t(\theta)^{1/2} \net(x;\theta) = M_t(\theta)^{1/2} \nabla \log p_t(x;\theta).$$

Recall also from \eqref{e:del_theta_net} that
\begin{align*}
&\partial_{\theta} \net (x,t,\phi^*(\theta))\\
=& \frac{1}{2} \sum_{i=1}^d \del_r \del_s \net(x+re_i+s\del_{\theta} {M_t(\theta)}  e_i,t,\phi^*(\theta)) \\
&+ \del_s \net (x+s\del_{\theta} {M_t(\theta)} \net(x,t,\phi^*(\theta)),t,\phi^*(\theta)).
\end{align*}

Combining the two equations above:
\begin{align*}
    &\del_\theta \flow(x;\phi^*(\theta)) \\
    =& M_t(\theta)^{1/2} \del_\theta \net(x;\phi^*(\theta)) + \frac{1}{2}M_t(\theta)^{-1/2} \lrp{\del_\theta M_t(\theta)} \net(x;\phi^*(\theta))\\
    =& \frac{1}{2} \sum_i M_t(\theta)^{1/2} \del_r \del_s \net (x+re_i+s\del_\theta {M_t(\theta)} e_i;\phi^*(\theta)) \\
    &+ M_t(\theta)^{1/2} \del_s \net (x+s\del_\theta {M_t(\theta)} \net(x;\phi^*(\theta));\phi^*(\theta)) + \frac{1}{2} M_t(\theta)^{-1/2} \lrp{\del_\theta M_t(\theta)} \net(x;\phi^*(\theta))\\
    =& \frac{1}{2} \sum_i \del_r \del_s \flow (x+re_i+s\del_\theta {M_t(\theta)} e_i;\phi^*(\theta)) \\
    &+ \del_s \flow (x+s M_t(\theta)^{-1/2} \del_\theta {M_t(\theta)}  \flow(x;\phi^*(\theta));\phi^*(\theta)) + \frac{1}{2} M_t(\theta)^{-1} \lrp{\del_\theta M_t(\theta)} \flow(x;\phi^*(\theta)).
\end{align*}

\section{Practical Choices of Matrix Schedule $M_t(\theta)$}
\label{s:implementation_details}

The framework in Sections~\ref{sec:pathwise-objective}--\ref{sec:inference_matrix} applies to arbitrary monotone PSD
matrix trajectories $M_t(\theta)$. Below, we demonstrate a number of specific parametric families of $M_t(\theta)$. In each case, we verify that the parameterization satisfies:
(i) \textbf{monotonicity} ($\partial_t M_t(\theta)\succ 0$),
(ii) \textbf{efficient evaluation} of matrix functions appearing in the loss and sampler
(e.g.\ $M_t^{\pm 1/2}$ and $(I+M_t)^{-1/2}$),
and (iii) \textbf{commutation of} $\partial_t M_t(\theta)$ and $M_t(\theta)$.

\subsection{A Subspace-Schedule Framework}
\label{ss:subspace_template}

We use a broad family based on an orthogonal decomposition of $\Re^d$ into subspaces. Let
$\{P_j\}_{j=1}^J$ be \emph{orthogonal projection matrices} such that $P_j=P_j^\top\succeq 0$,
$P_j^2=P_j$, $P_jP_k=0$ for $j\neq k$, and $\sum_{j=1}^J P_j=I$. Let
$\{g_j(\cdot;\theta)\}_{j=1}^J$ be strictly increasing scalar functions with $g_j(0;\theta)=g_0$ and $g_j(T;\theta)=T$, where $g_0>0$ is a small constant, so that $M_0=g_0I\approx 0$. In our experiments, each $g_j(\cdot;\theta)$ is parameterized by a monotone knot-based family (Appendix \ref{ss:g_i_impl_appendix}), which ensures $\dot g_j>0$ and differentiability in $\theta$.
We define the matrix-valued schedule
\begin{equation}
M_t(\theta) := \sum_{j=1}^J g_j(t;\theta)\,P_j,
\label{e:subspace_M}
\end{equation}
so that $\partial_t M_t(\theta)=\sum_{j=1}^J \dot g_j(t;\theta)P_j\succ 0$ whenever $\dot g_j>0$ for all $j$.
This construction therefore satisfies the monotonicity constraint in \eqref{e:Mt-psd} automatically. Because the $P_j$ are mutually orthogonal projectors, any (scalar) matrix function $f$ acts \emph{subspace-wise}:
\begin{equation}
f(M_t(\theta))=\sum_{j=1}^J f(g_j(t;\theta))\,P_j.
\label{e:functional_calculus}
\end{equation}
In particular, the quantities needed by our loss and sampler ($M_t^{\pm 1/2}$, $M_t^{-1}$, and
$(I+M_t)^{-1/2}$) reduce to inexpensive scalings on each subspace. This also matches the solver structure in
Section~\ref{sec:inference_matrix}, where updates depend on increments of $M_t^{1/2}$; under \eqref{e:subspace_M},
these increments are simply subspace-wise scalar increments.

The form of \eqref{e:subspace_M} separates:
\begin{itemize}[leftmargin=*,nosep]
\item \textbf{Directions / subspaces} (the projectors $\{P_j\}$, determined by a basis choice and a grouping), and
\item \textbf{Time allocation} (the scalar schedules $\{g_j(\cdot;\theta)\}$).
\end{itemize}
This separation lets us trade off interpretability and expressivity while keeping training and sampling efficient.

\subsection{Schedule Families Used in Experiments}
\label{ss:schedule_families}
We now list the concrete instances of \eqref{e:subspace_M} used in our experiments in Section~\ref{s:experiment}.

\paragraph{(a) Isotropic schedule (warm-up / special case).}
Standard isotropic diffusion corresponds to $J=1$ with $P_1=I$, so $M_t(\theta)=g(t;\theta)I$.
This serves as a sanity check and warm-up baseline; it also emphasizes that our framework strictly generalizes scalar
schedule learning.

\paragraph{(b) Fixed structured basis (DCT subspaces).}
For images, we often choose a fixed orthonormal transform $U\in\Re^{d\times d}$ exposing interpretable directions
(e.g.\ a 2D-DCT basis; see Appendix \ref{ss:2d_dct_appendix}). Partition the basis indices into $J$ groups defining subspaces
$S_1,\dots,S_J$ (ordered by increasing frequency), and let $P_j$ be the projector onto $S_j$.
A particularly simple instance is $J=2$:
$S_1$ is a low-frequency subspace and $S_2=S_1^\perp$ its complement, giving
\[
M_t(\theta)=g^{\mathrm{DCT}}_1(t;\theta)P_1 + g^{\mathrm{DCT}}_2(t;\theta)P_2.
\]
This parameterization yields an interpretable notion of anisotropy (coarse-to-fine denoising) while remaining cheap.

\paragraph{(c) Class-conditional PCA bases (data-dependent subspaces, shared schedules).}
To incorporate data geometry, we can make the \emph{basis} depend on a class label $y\in\{1,\dots,C\}$ (in the case of ImageNet and CIFAR10).
For each class $y$, compute a PCA basis from training samples in that class and form mutually orthogonal subspaces
$\{S^{(y)}_j\}_{j=1}^J$ that partition $\Re^d$ (e.g.\ $S^{(y)}_1$ spans the top-$k$ PCs and the remaining subspaces
capture residual directions). Let $P^{(y)}_j$ be the projector onto $S^{(y)}_j$. We then define
\[
M_t(\theta;y)=\sum_{j=1}^J g^{\mathrm{PCA}}_j(t;\theta)\,P^{(y)}_j.
\]
Here the \emph{subspaces} vary with $y$ while the scalar schedules are shared across classes, isolating the effect of a
class-dependent geometry.

\paragraph{(d) Shared structured basis + class-conditional schedules.}
For class-conditional generation with a score network $\net(x,t,y;\phi)$, we can instead keep the projectors fixed
(e.g.\ DCT projectors $\{P_j\}$) and make the schedules class-specific:
\[
M_t(\theta;y)=\sum_{j=1}^J g^{\mathrm{DCT}}_{j,y}(t;\theta)\,P_j.
\]
Even in the simple $J=2$ case this yields $2C$ coupled scalar trajectories, making black-box search over $\theta$ impractical and
motivating our gradient-based optimization of $\theta$.

\paragraph{(e) Class-conditional PCA bases + class-conditional schedules (most expressive).}
Finally, we can let \emph{both} the subspaces and schedules depend on class:
\[
M_t(\theta;y)=\sum_{j=1}^J g^{\mathrm{PCA}}_{j,y}(t;\theta)\,P^{(y)}_j.
\]
This is the largest design space we consider. It can align anisotropic diffusion dynamics with class-dependent data
geometry, but it also highlights why an efficient variational optimization scheme is needed.

\paragraph{Discussion: expressivity vs.\ efficiency.}
Across (a)–(e), the template \eqref{e:subspace_M} ensures that all matrix operations required by the trajectory loss and
the anisotropic solvers reduce to subspace-wise scalings via \eqref{e:functional_calculus}. Increasing expressivity
corresponds to increasing $J$, allowing class dependence, or making the basis data-dependent (PCA), while retaining
tractable computation.

\subsection{Schedule Variants Evaluated}
\label{ss:exp_variants}

We compare the following schedule families (see Section~\ref{s:implementation_details} for the corresponding
parameterizations):
\begin{enumerate}[leftmargin=*,nosep]
    \item \textbf{Class-conditional isotropic schedule $\{g_y\}_{y=1}^C$.}
    A separate scalar schedule per class:
    $M_t(y)=g_y(t;\theta)I$.
    \item \textbf{DCT anisotropic schedule $(g^{\mathrm{DCT}}_1,g^{\mathrm{DCT}}_2)$.}
    A two-subspace matrix schedule with $J=2$ in a fixed 2D-DCT basis (Appendix~\ref{ss:2d_dct_appendix}), where the ``low-frequency''
    subspace consists of the $(H/2)\times(H/2)$ lowest-frequency DCT modes (i.e., $H^2/4$ basis vectors) and the
    second subspace is its orthogonal complement:
    $M_t=g^{\mathrm{DCT}}_1(t;\theta)P_1+g^{\mathrm{DCT}}_2(t;\theta)P_2$.
    \item \textbf{PCA anisotropic schedule $(g^{\mathrm{PCA}}_1,g^{\mathrm{PCA}}_2)$.}
    A two-subspace matrix schedule with $J=2$ using class-conditional PCA bases (subspaces depend on class, schedules
    are shared), as described in Section~\ref{s:implementation_details}.
    \item \textbf{Class-conditional DCT anisotropic schedules $\{(g^{\mathrm{DCT}}_{1,y},g^{\mathrm{DCT}}_{2,y})\}_{y=1}^C$.}
    Fixed DCT subspaces but class-specific scalar schedules.
    \item \textbf{Class-conditional PCA anisotropic schedules $\{(g^{\mathrm{PCA}}_{1,y},g^{\mathrm{PCA}}_{2,y})\}_{y=1}^C$.}
    The most expressive conditional variant we consider: both PCA subspaces and schedules depend on class.
\end{enumerate}

\paragraph{Conditional vs.\ unconditional datasets.}
We consider CIFAR-10 and ImageNet-64 as class-conditional datasets. For CIFAR-10, which is relatively simple, we first train isotropic schedules, including both a global isotropic schedule and class-conditional isotropic schedules, as warm-up ablations. Building on this, we separately study the effects of class-conditional schedules and class-conditional bases. For ImageNet-64, due to its higher visual complexity, we directly adopt class-conditional isotropic schedules during warm-up. We then evaluate more expressive variants, including class-conditional schedules, class-conditional bases, and their combination. For AFHQv2 and FFHQ, which do not provide class labels, we primarily train structured-basis anisotropic schedules based on a fixed DCT decomposition with a global isotropic schedule $g$ as an ablation for warm-up.

% =====================================================================
\section{Experiment Details}
\label{s:impl}
% =====================================================================

\subsection{Subspace Schedule Parameterization $g_i$}
\label{ss:g_i_impl_appendix}
For a fixed set of $m$ node locations $0=\tau_0 < \tau_1 < \cdots < \tau_{m-1}=T,$ the subspace noise schedule is defined in log--space between the smallest and largest variance values $g(0)=g_0$ and $g(T)=T$. The trainable parameters $\theta = (\theta_1,\dots,\theta_{m-1})$ define a collection of strictly positive increments
$$
s_i = \operatorname{softplus}(\theta_i), \qquad i=1,\dots,m-1.
$$
These increments are then rescaled by $\alpha$ so that their sum exactly matches the total log--gap between the endpoints,
$$
\alpha = \frac{\log T - \log g_0}{\sum_{i=1}^{m-1} s_i}.
$$
The log--values at the nodes are constructed by cumulative summation,
$$
\ell_0 = \log g_0, \qquad
\ell_j = \ell_0 + \sum_{i=1}^{j} \alpha s_i, \quad j=1,\dots,m-1,
$$
so that $\ell_{m-1} = \log T$.

Given a time $t\in[0,T]$, one locates the enclosing interval $[\tau_{j-1},\tau_j]$ and computes the normalized position
$$
p(t) = \frac{t-\tau_{j-1}}{\tau_j - \tau_{j-1}}.
$$
The value of $g(t)$ is then obtained by linearly interpolating between successive log--nodes and exponentiating:
$$
\log g(t) = (1-p(t))\ell_{j-1} + p(t)\ell_j,
\qquad
g(t) = \exp\!\big(\log g(t)\big).
$$
Within each interval the derivative takes the simple form
$$
g'(t) = g(t)\frac{\ell_j - \ell_{j-1}}{\tau_j - \tau_{j-1}}.
$$

\subsection{Training Setup}
\label{ss:training_setup}

\paragraph{Initialization.}
All experiments are initialized from the corresponding pretrained EDM checkpoints.
We load both raw network weights and their exponential moving average (EMA) counterparts.
When learning diffusion schedules, we initialize schedule parameters from a pretrained
isotropic schedule model.
For conditional variants, schedule parameters are instantiated per class.
If class-wise isotropic schedules are available, they are used for initialization;
otherwise, a shared global schedule is copied across classes.

\paragraph{Joint optimization of network and schedule.}
We jointly optimize network parameters $\phi$ and schedule parameters $\theta$
using Adam ($\beta_1=0.9$, $\beta_2=0.999$, $\epsilon=10^{-8}$).
Network and schedule parameters use separate learning rates.
Gradients are accumulated over multiple micro-batches to achieve the target effective batch size.
After backpropagation, gradients are explicitly sanitized by replacing NaNs
and infinities with finite values to ensure numerical stability.

For class-conditional schedules, schedule updates are applied only to classes
present in the current batch; in distributed settings, we synchronize the set
of active classes across workers before applying optimizer steps to maintain
consistent parameter updates.

\paragraph{Learning-rate schedule.}
We employ a linear warm-up for the first $K_{\mathrm{ramp}}$ kimg. After warm-up, the learning rate follows a mild decay schedule, with a decay factor applied after the midpoint of training. The learning rate of schedule parameters scales proportionally to the network learning rate.

\paragraph{EMA update.}
We maintain an exponential moving average of network parameters throughout training. The EMA half-life is specified in kimg units. Optionally, we apply an EMA ramp-up, where the effective half-life is bounded by the number of images seen so far during early training. EMA parameters are synchronized across workers in distributed settings.

\paragraph{Distributed training.}
For large-scale experiments (e.g., ImageNet-64), we use PyTorch Distributed Data Parallel (DDP)
with NCCL backend. The global batch is evenly partitioned across workers. All parameters are broadcast from rank 0 at initialization. Gradients of network parameters are synchronized via DDP. Schedule parameters are explicitly all-reduced across workers, including zero-filled gradients for classes absent on a given worker, to guarantee consistent updates across GPUs.

\paragraph{Training budget and checkpoints.}
Across all datasets, fine-tuning processes approximately $1.2$M images in total.
We periodically save raw network weights, EMA weights, and schedule parameters.
Unless otherwise specified, evaluation is performed using the final EMA checkpoint.

\subsection{Evaluation protocol details}
\label{ss:eval_protocol}

\paragraph{NFE counting convention.}
Both Heun variants evaluate the network twice per step. EDM uses the
\emph{endpoint variant} $\hat t_k = t_{k-1}$ and skips the second-order
correction on its final step to $\sigma=0$, so $N$ steps cost NFE $=2N-1$
\citep{karras2022elucidating}. Our matrix Heun solver uses the \emph{midpoint
variant} $\hat t_k = (t_{k-1}+t_k)/2$; since $\hat t_k>0$ on every step, all
steps, including the last, use two evaluations. At each paired budget in
Table~\ref{tab:fid-all}, we run $N-1$ midpoint steps, giving NFE $=2N-2$, one
fewer than EDM. Throughout the paper, we report each method under its own
counting convention.

\subsection{Composing Anisotropic Schedules with DPM-Solver-v3}
\label{s:dpmv3_composition}
 
This appendix explains how our learned anisotropic schedule $M_t(\theta)$ composes with off-the-shelf scalar diffusion solvers such as DPM-Solver and DPM-Solver-v3~\citep{lu2022dpm,zheng2023dpm}, yielding the \emph{Ours\,+\,DPM-v3} sampler reported in Table~\ref{tab:fid-extra} (right).
 
\subsubsection{Subspace decoupling of the anisotropic dynamics}
\label{ss:dpmv3_decoupling}
 
Recall the projector-based parameterization (Section~\ref{s:implementation_details}):
\begin{equation}
  M_t(\theta) = \sum_{j=1}^{J} g_j(t;\theta)\, P_j,
  \qquad P_i P_j = \delta_{ij}P_j,\qquad \sum_j P_j = I.
  \label{e:dpmv3_Mt}
\end{equation}
Let $V_j$ be an orthonormal basis spanning $\mathrm{range}(P_j)$, so that $P_j = V_j V_j^\top$, and let $V := [V_1\,\cdots\,V_J]$ collect the bases. The forward dynamics $dx_t = (\partial_t M_t)^{1/2}\,dB_t$ commutes with $V$ in the sense that the projected coordinate $z^{(j)}_t := V_j^\top x_t$ satisfies a \emph{scalar} variance-exploding diffusion within subspace $j$:
\begin{equation}
  d z^{(j)}_t = \big(\partial_t g_j(t;\theta)\big)^{1/2} dB^{(j)}_t,
  \qquad
  z^{(j)}_t \mid z^{(j)}_0 \sim \mathcal{N}\!\big(z^{(j)}_0,\, g_j(t;\theta)\,I_{d_j}\big),
  \label{e:dpmv3_subspace_dyn}
\end{equation}
where $d_j = \mathrm{rank}(P_j)$ and $B^{(j)}$ are independent Brownian motions.
The effective scalar noise level in subspace $j$ is therefore $\sigma^{(j)}(t)^2 = g_j(t;\theta)$.
 
\subsubsection{Per-subspace solver application}
\label{ss:dpmv3_application}
 
Equation~\eqref{e:dpmv3_subspace_dyn} states that, in the $V$-frame, our anisotropic process is exactly $J$ \emph{independent} scalar VE diffusions sharing a common time variable. Consequently, any scalar diffusion solver (e.g., DPM-Solver, DPM-Solver++, DPM-Solver-v3, or EDM Heun) applies within each subspace using the corresponding scalar schedule $g_j(t;\theta)$. Concretely, given a learned score field $\nabla_x\log p_t(x;\theta)$ approximated by the trained network, one DPM-Solver-v3 step within subspace $j$ uses:
\begin{itemize}[leftmargin=*,nosep]
  \item the scalar schedule $\sigma^{(j)}(t) = g_j(t;\theta)^{1/2}$ to construct the solver's internal log-SNR grid and update coefficients;
  \item the projected score $V_j^\top \nabla_x\log p_t(x;\theta)$ as the solver's denoiser output;
  \item the projected iterate $V_j^\top x$ as the state being updated.
\end{itemize}
The composed sampler $\textsc{Ours\,+\,DPM-v3}$ is then simply: (i) project the iterate into the $V$-frame, (ii) take a DPM-Solver-v3 step in each subspace using $g_j(t;\theta)$, (iii) project back. The procedure inherits the per-step cost of one network call (which produces the score field $\nabla_x \log p_t$) plus inexpensive projections onto $\{V_j\}$.
 
\subsubsection{Implementation note: \emph{scalarized} mode}
\label{ss:dpmv3_scalarized}
 
Our implementation realizes the per-subspace step in a particularly lightweight way that we refer to as the \emph{scalarized} mode (used in our codebase as \texttt{dpmv3\_official\_scalar\_eff}):
\begin{enumerate}[leftmargin=*,nosep]
  \item evaluate the score network on the un-projected iterate $x$, yielding $s := \nabla_x\log p_t(x;\theta)$;
  \item form the DPM-Solver-v3 update independently for each subspace $j$ using the scalar schedule $\sigma^{(j)}(t) = g_j(t;\theta)^{1/2}$ and the projected score $V_j^\top s$;
  \item assemble the next-step iterate by summing contributions across subspaces, $x' = \sum_j V_j (V_j^\top x'_j)$, where $x'_j$ is the per-subspace update.
\end{enumerate}
This avoids modifying the score network's input space and allows direct reuse of the EDM-style network used by our matrix Heun solver. The reduction is exact under~\eqref{e:dpmv3_Mt} because the dynamics, the marginals, and the score all decouple along $V$; ``scalarized'' refers only to the fact that the solver's internal step computation is performed using the scalar effective $\sigma^{(j)}(t)$ rather than the full matrix $M_t^{1/2}$.
 
\subsubsection{When to use which solver}
\label{ss:dpmv3_regimes}
 
In Table~\ref{tab:fid-extra} (right), \emph{Ours\,+\,DPM-v3} dominates at NFE $\le 8$, while our matrix Heun solver of Section~\ref{sec:inference_matrix} dominates at NFE $\ge 34$. The crossover reflects standard tradeoffs in diffusion sampling: DPM-Solver-v3 uses higher-order extrapolation tuned for very small step counts, whereas the second-order Heun solver becomes more accurate as the discretization becomes finer. For practitioners, we recommend the composed sampler for fast inference and the matrix Heun solver when fidelity is the priority.

% =====================================================================
\section{Additional Experimental Analysis}
\label{s:additional_exp}
% =====================================================================
\subsection{Per-Variant FID Breakdown}
\label{s:per_variant_results}
 
Table~\ref{tab:fid-all-error-bar} provides the full per-variant breakdown across all evaluated NFE budgets. Rows are labeled by schedule family (Appendix~\ref{ss:exp_variants}):
\begin{itemize}[leftmargin=*,nosep]
    \item $g$: global isotropic schedule;
    \item $\{g_y\}_{y=1}^C$: class-conditional isotropic schedules;
    \item $(g^{\mathrm{DCT}}_1,g^{\mathrm{DCT}}_2)$: two-subspace DCT anisotropic schedule;
    \item $(g^{\mathrm{PCA}}_1,g^{\mathrm{PCA}}_2)$: two-subspace PCA anisotropic schedule with class-conditional bases and shared schedules;
    \item $\{(g^{\mathrm{DCT}}_{1,y},g^{\mathrm{DCT}}_{2,y})\}_{y=1}^C$: class-conditional DCT schedules;
    \item $\{(g^{\mathrm{PCA}}_{1,y},g^{\mathrm{PCA}}_{2,y})\}_{y=1}^C$: class-conditional PCA bases and schedules.
\end{itemize}

\begin{table*}[ht]
\centering
\normalsize
\setlength{\tabcolsep}{6pt}
\begin{tabular}{l|cccccccc}
\toprule

\multicolumn{9}{c}{\textbf{CIFAR-10}} \\
\midrule
NFE (EDM/Ours) & 9/8 & 11/10 & 13/12 & 15/14 & 17/16 & 35/34 & 59/58 & 79/78 \\
\midrule
EDM 
& \fidstd{35.52}{0.0448} & \fidstd{14.37}{0.0387} & \fidstd{6.694}{0.0563} & \fidstd{4.231}{0.0462}
& \fidstd{3.027}{0.0436} & \fidstd{1.829}{0.0162} & \fidstd{1.868}{0.0127} & \fidstd{1.890}{0.0112} \\
$g$ 
& \fidstd{4.241}{0.0590} & \fidstd{3.087}{0.0163} & \fidstd{2.650}{0.0120} & \fidstd{2.431}{0.0177}
& \fidstd{2.328}{0.0166} & \fidstd{\textbf{1.815}}{0.0305} & \fidstd{1.820}{0.0299} & \fidstd{1.835}{0.0302} \\
$\{g_y\}_{y=1}^C$ 
& \fidstd{3.691}{0.0165} & \fidstd{3.090}{0.0159} & \fidstd{2.735}{0.0112} & \fidstd{2.582}{0.0069}
& \fidstd{2.480}{0.0096} & \fidstd{1.831}{0.0145} & \fidstd{1.829}{0.0081} & \fidstd{1.841}{0.0080} \\
$\{(g^{\mathrm{DCT}}_{1,y}, g^{\mathrm{DCT}}_{2,y})\}_{y=1}^C$ 
& \fidstd{\textbf{3.618}}{0.0202} & \fidstd{3.056}{0.0152} & \fidstd{2.681}{0.0118} & \fidstd{2.535}{0.0061}
& \fidstd{2.431}{0.0065} & \fidstd{1.829}{0.0120} & \fidstd{1.810}{0.0072} & \fidstd{1.818}{0.0078} \\
$(g^{\mathrm{PCA}}_1, g^{\mathrm{PCA}}_2)$ 
& \fidstd{3.839}{0.0661} & \fidstd{\textbf{2.900}}{0.0243} & \fidstd{\textbf{2.538}}{0.0112} & \fidstd{\textbf{2.332}}{0.0029}
& \fidstd{\textbf{2.213}}{0.0062} & \fidstd{1.851}{0.0113} & \fidstd{\textbf{1.803}}{0.0103} & \fidstd{\textbf{1.803}}{0.0138} \\
\midrule

\multicolumn{9}{c}{\textbf{AFHQv2}} \\
\midrule
NFE (EDM/Ours) & 9/8 & 11/10 & 13/12 & 15/14 & 19/18 & 39/38 & 79/78 & 119/118  \\
\midrule
EDM 
& \fidstd{27.98}{0.0544} & \fidstd{13.66}{0.0778} & \fidstd{7.587}{0.0729} & \fidstd{4.746}{0.0486}
& \fidstd{2.986}{0.0359} & \fidstd{2.075}{0.0167} & \fidstd{2.042}{0.0175} & \fidstd{2.046}{0.0173} \\
$(g^{\mathrm{DCT}}_1, g^{\mathrm{DCT}}_2)$ 
& \fidstd{\textbf{4.524}}{0.0210} & \fidstd{\textbf{2.977}}{0.0195} & \fidstd{\textbf{2.513}}{0.0108} & \fidstd{\textbf{2.401}}{0.0026}
& \fidstd{\textbf{2.265}}{0.0134} & \fidstd{\textbf{2.070}}{0.0094} & \fidstd{\textbf{2.015}}{0.0095} & \fidstd{\textbf{2.010}}{0.0255} \\
\midrule

\multicolumn{9}{c}{\textbf{FFHQ}} \\
\midrule
NFE (EDM/Ours) & 9/8 & 11/10 & 13/12 & 15/14 & 19/18 & 39/38 & 79/78 & 119/118  \\
\midrule
EDM 
& \fidstd{57.14}{0.0774} & \fidstd{29.39}{0.0497} & \fidstd{15.81}{0.0857} & \fidstd{9.769}{0.0853}
& \fidstd{5.169}{0.0480} & \fidstd{2.575}{0.0352} & \fidstd{2.391}{0.0374} & \fidstd{2.374}{0.0381} \\
$(g^{\mathrm{DCT}}_1, g^{\mathrm{DCT}}_2)$ 
& \fidstd{\textbf{4.967}}{0.0400} & \fidstd{\textbf{3.776}}{0.0321} & \fidstd{\textbf{3.304}}{0.0263} & \fidstd{\textbf{3.011}}{0.0351}
& \fidstd{\textbf{2.711}}{0.0272} & \fidstd{\textbf{2.350}}{0.0282} & \fidstd{\textbf{2.297}}{0.0331} & \fidstd{\textbf{2.318}}{0.0336} \\
\midrule

\multicolumn{9}{c}{\textbf{ImageNet-64}} \\
\midrule
NFE (EDM/Ours) & 9/8 & 11/10 & 13/12 & 15/14 & 19/18 & 39/38 & 79/78 & 119/118  \\
\midrule
EDM 
& \fidstd{35.32}{0.2144} & \fidstd{15.19}{0.0894} & \fidstd{8.160}{0.0191} & \fidstd{5.458}{0.0468}
& \fidstd{3.586}{0.0533} & \fidstd{2.428}{0.0453} & \fidstd{2.294}{0.0350} & \fidstd{2.276}{0.0326} \\
$\{g_y\}_{y=1}^C$  
& \fidstd{8.416}{0.0574} & \fidstd{5.907}{0.0224} & \fidstd{4.749}{0.0364} & \fidstd{4.079}{0.0347}
& \fidstd{3.366}{0.0327} & \fidstd{2.455}{0.0090} & \fidstd{2.314}{0.0116} & \fidstd{2.296}{0.0095} \\
$(g^{\mathrm{PCA}}_1, g^{\mathrm{PCA}}_2)$ 
& \fidstd{29.47}{0.1691} & \fidstd{15.55}{0.1727} & \fidstd{8.558}{0.1230} & \fidstd{5.666}{0.0867}
& \fidstd{3.409}{0.0619} & \fidstd{\textbf{2.362}}{0.0498} & \fidstd{2.263}{0.0446} & \fidstd{2.240}{0.0464} \\
$\{(g^{\mathrm{PCA}}_{1,y}, g^{\mathrm{PCA}}_{2,y})\}_{y=1}^C$ 
& \fidstd{8.366}{0.0492} & \fidstd{5.823}{0.0545} & \fidstd{4.626}{0.0786} & \fidstd{3.992}{0.0564}
& \fidstd{3.282}{0.0575} & \fidstd{2.378}{0.0539} & \fidstd{\textbf{2.259}}{0.0416} & \fidstd{2.239}{0.0396} \\
$\{(g^{\mathrm{DCT}}_{1,y}, g^{\mathrm{DCT}}_{2,y})\}_{y=1}^C$ 
& \fidstd{\textbf{8.340}}{0.0723} & \fidstd{\textbf{5.808}}{0.0601} & \fidstd{\textbf{4.615}}{0.0833} & \fidstd{\textbf{3.983}}{0.0592}
& \fidstd{\textbf{3.275}}{0.0586} & \fidstd{2.377}{0.0546} & \fidstd{\textbf{2.259}}{0.0414} & \fidstd{\textbf{2.238}}{0.0412} \\
\bottomrule
\end{tabular}

\vspace{0.38em}

\caption{FID $\downarrow$ vs.\ NFE across datasets (50k samples). See Appendix~\ref{ss:exp_variants} for a description of each schedule variant. Each entry reports the minimum FID across 3 independent random generations of 50k images, with the standard deviation shown in parentheses.}
\label{tab:fid-all-error-bar}
\end{table*}
% =====================================================================

\subsection{Learned Schedules Across Datasets}
\label{ss:learned_schedules}

Figures~\ref{fig:cifar10_schedule}--\ref{fig:imagenet_schedule} visualize the learned schedules across datasets, subspace parameterizations, and conditioning schemes.

First, across all datasets (Figures~\ref{fig:cifar10_schedule}--\ref{fig:ffhq_schedule}), the learned anisotropic ratios deviate systematically from one over a substantial portion of the diffusion horizon, confirming that the two subspaces are not treated symmetrically by the learned schedule.
However, the magnitude and shape of this deviation depend strongly on the dataset.
For CIFAR-10 (Figure~\ref{fig:cifar10_schedule}(b)) and AFHQv2 (Figure~\ref{fig:afhqv2_schedule}), the anisotropy ratio exhibits a smooth mid-horizon amplification before returning to one near the endpoints, suggesting that anisotropy is primarily active during intermediate diffusion times.
In contrast, for FFHQ (Figure~\ref{fig:ffhq_schedule}) the ratio shows a more pronounced early amplification followed by a gradual decay, indicating stronger low-frequency emphasis at early noise levels.

Second, on the class-conditional datasets (CIFAR-10 and ImageNet-64; Figures~\ref{fig:cifar10_schedule} and~\ref{fig:imagenet_schedule}), class-conditional training induces clear inter-class variability in both isotropic and anisotropic schedules.
In CIFAR-10 (Figure~\ref{fig:cifar10_schedule}(c,d)), the class-conditional isotropic schedules display substantial dispersion around the geometric mean, and the anisotropic schedules further amplify these differences.
On ImageNet-64 (Figure~\ref{fig:imagenet_schedule}(a--f)), this effect becomes more structured: while class-conditional isotropic schedules vary moderately, the anisotropic variants reveal consistent class-dependent shifts in both geometric means and anisotropy ratios.
Notably, the PCA-based conditional schedules in Figure~\ref{fig:imagenet_schedule}(d--f) exhibit tighter concentration than their DCT counterparts (Figure~\ref{fig:imagenet_schedule}(b,c)), suggesting that class-adaptive subspaces reduce unnecessary variability.

\begin{figure*}[ht]
    \centering
    \includegraphics[width=0.8\linewidth]{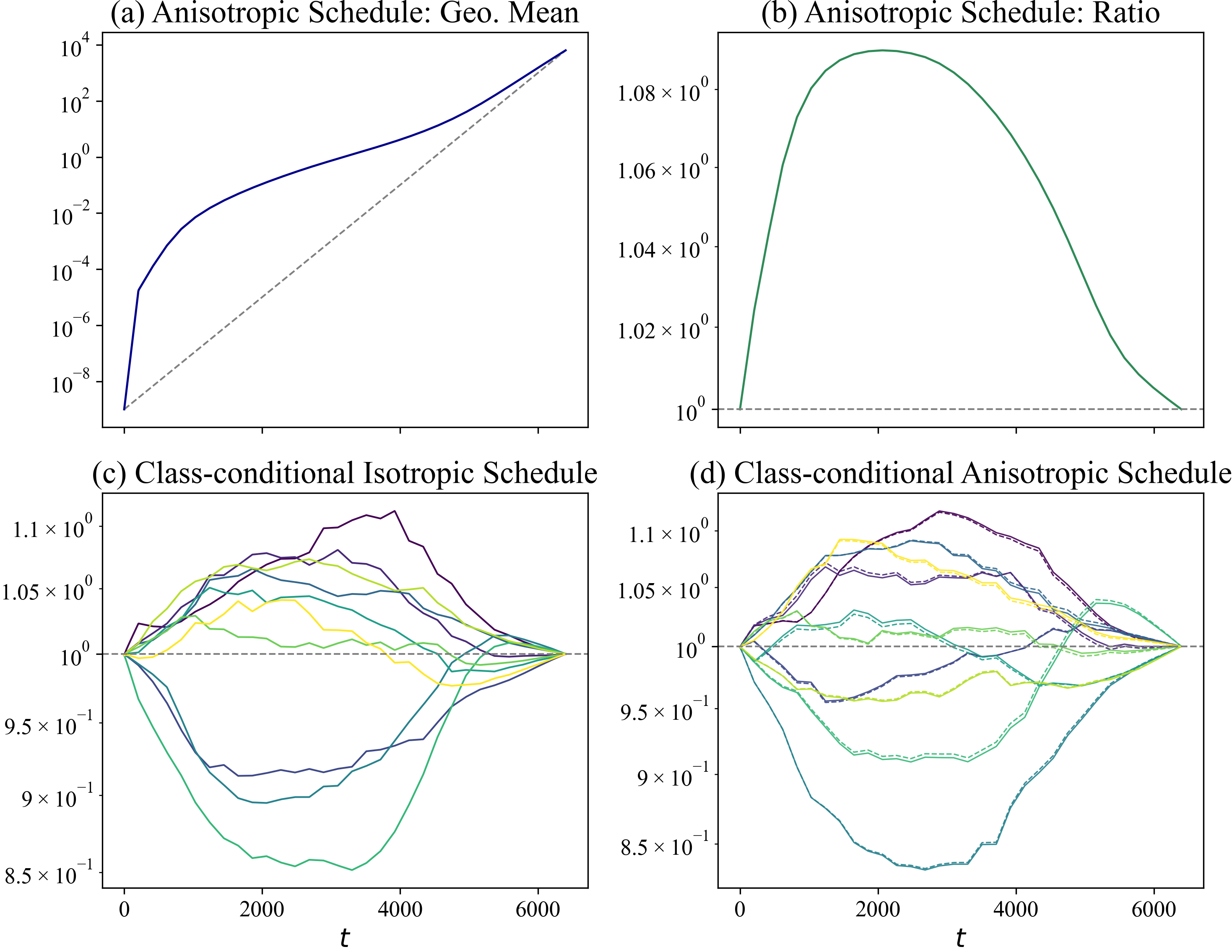}
    \caption{\textbf{CIFAR-10 learned schedule analysis.}
(a) PCA-based geometric mean $\sqrt{g^{\mathrm{PCA}}_{1}(t)g^{\mathrm{PCA}}_{2}(t)}$ with a log-linear reference (gray dashed; log $y$-axis).
(b) PCA anisotropy ratio $g^{\mathrm{PCA}}_{1}(t)/g^{\mathrm{PCA}}_{2}(t)$.
(c) Class-conditional isotropic schedules $g_y(t)/\bar g(t)$, where $\bar g(t)=\big(\prod_{y=1}^{C} g_y(t)\big)^{1/C}$ with $C=\#\text{class}=10$.
(d) Class-conditional anisotropic schedules over DCT subspaces:
$g^{\mathrm{DCT}}_{k,y}(t)/\bar g_k(t)$ for $k\in\{1,2\}$ (solid/dashed),
where $y$ indexes classes and $\bar g_k(t)$ is the geometric mean across classes.}
    \label{fig:cifar10_schedule}
\end{figure*}

\begin{figure*}[ht]
    \centering
    \includegraphics[width=0.3\linewidth]{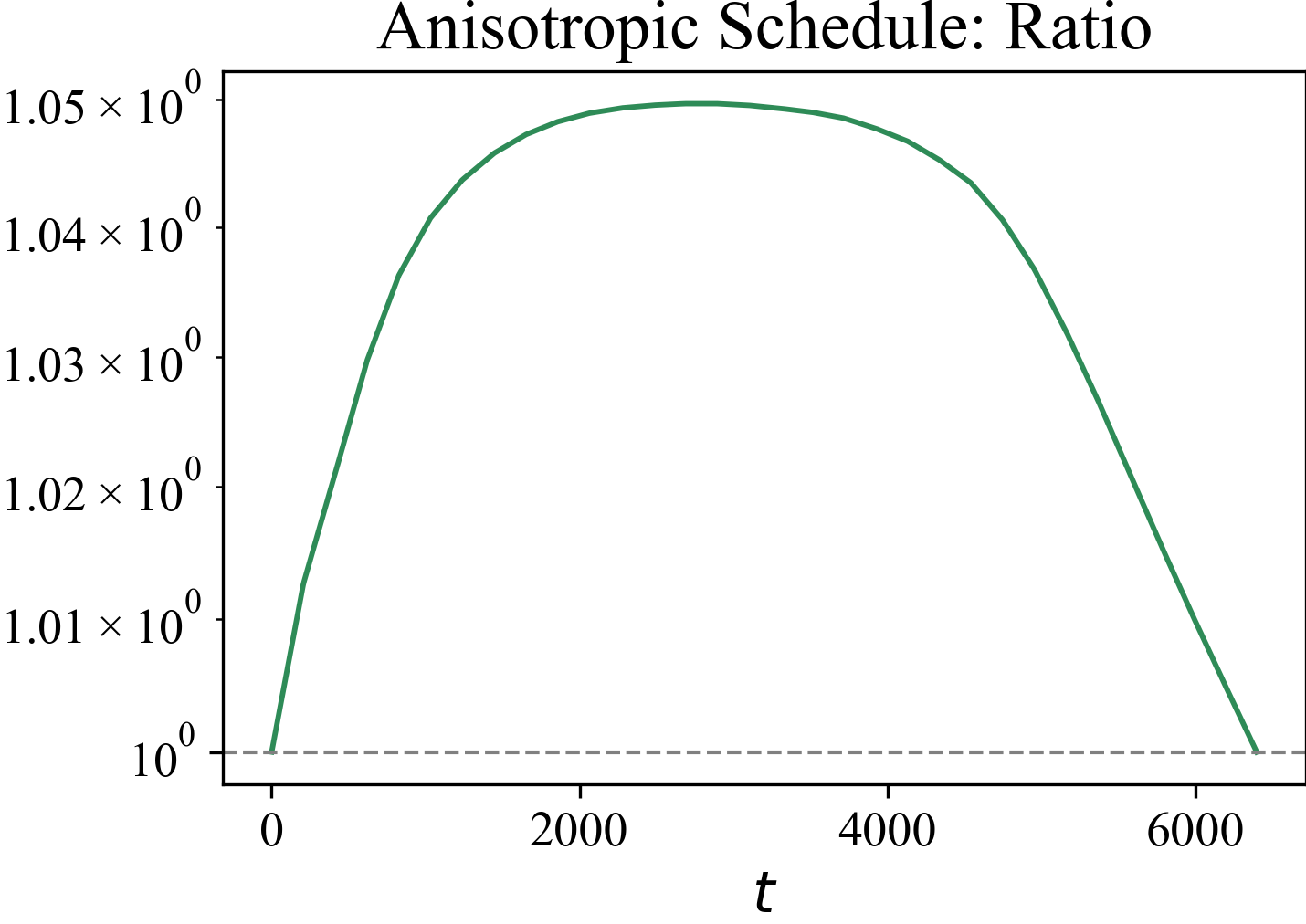}
    \caption{\textbf{AFHQv2 learned schedule analysis.}
Ratio between the two learned DCT-based schedules $g^{\mathrm{DCT}}_1(t)$ and $g^{\mathrm{DCT}}_2(t)$ over time.}
    \label{fig:afhqv2_schedule}
\end{figure*}

\begin{figure*}[ht]
    \centering
    \includegraphics[width=0.3\linewidth]{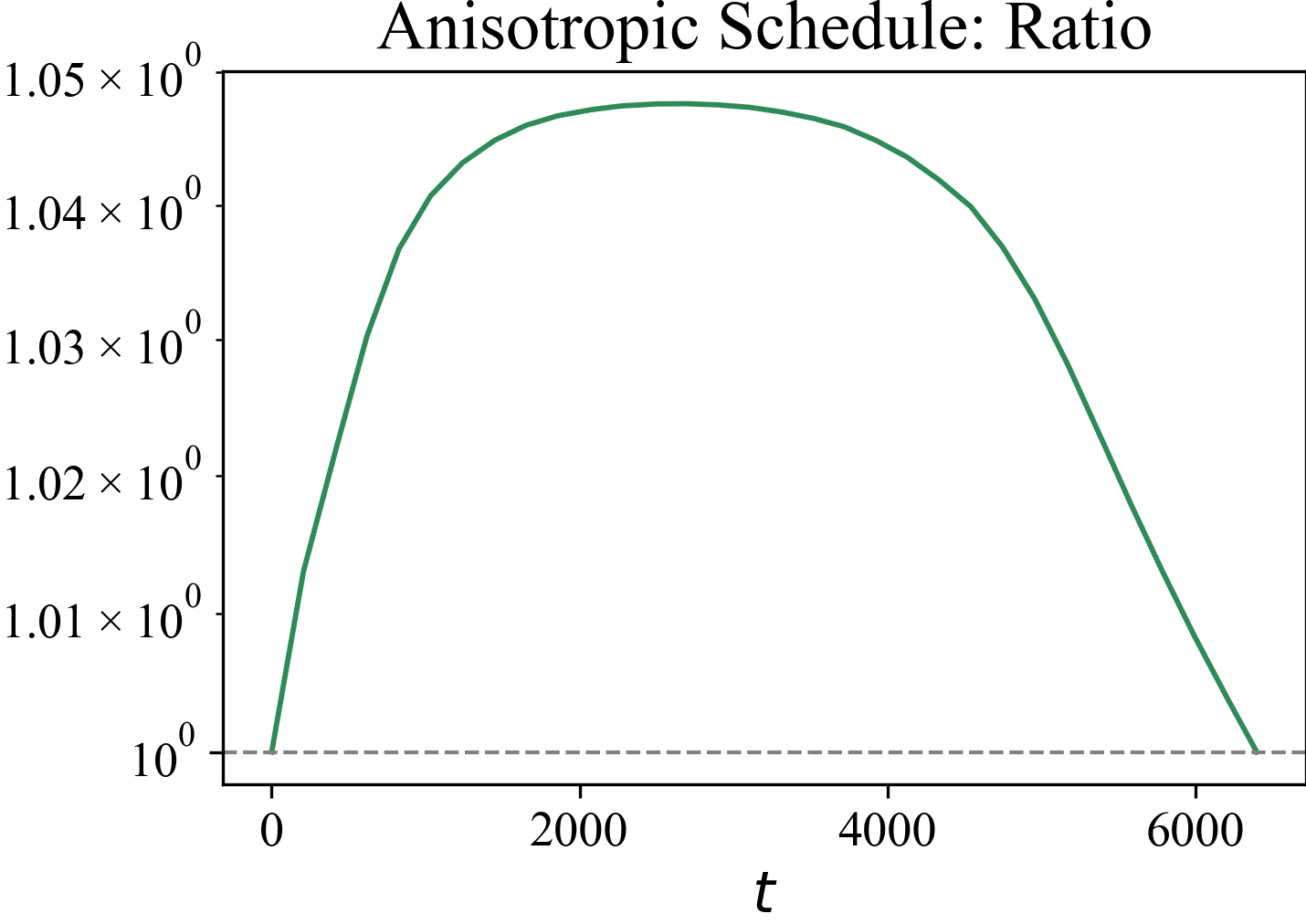}
    \caption{\textbf{FFHQ learned schedule analysis.}
Ratio between the two learned DCT-based schedules $g^{\mathrm{DCT}}_1(t)$ and $g^{\mathrm{DCT}}_2(t)$ over time.}
    \label{fig:ffhq_schedule}
\end{figure*}

\begin{figure*}[ht]
    \centering
    \includegraphics[width=\linewidth]{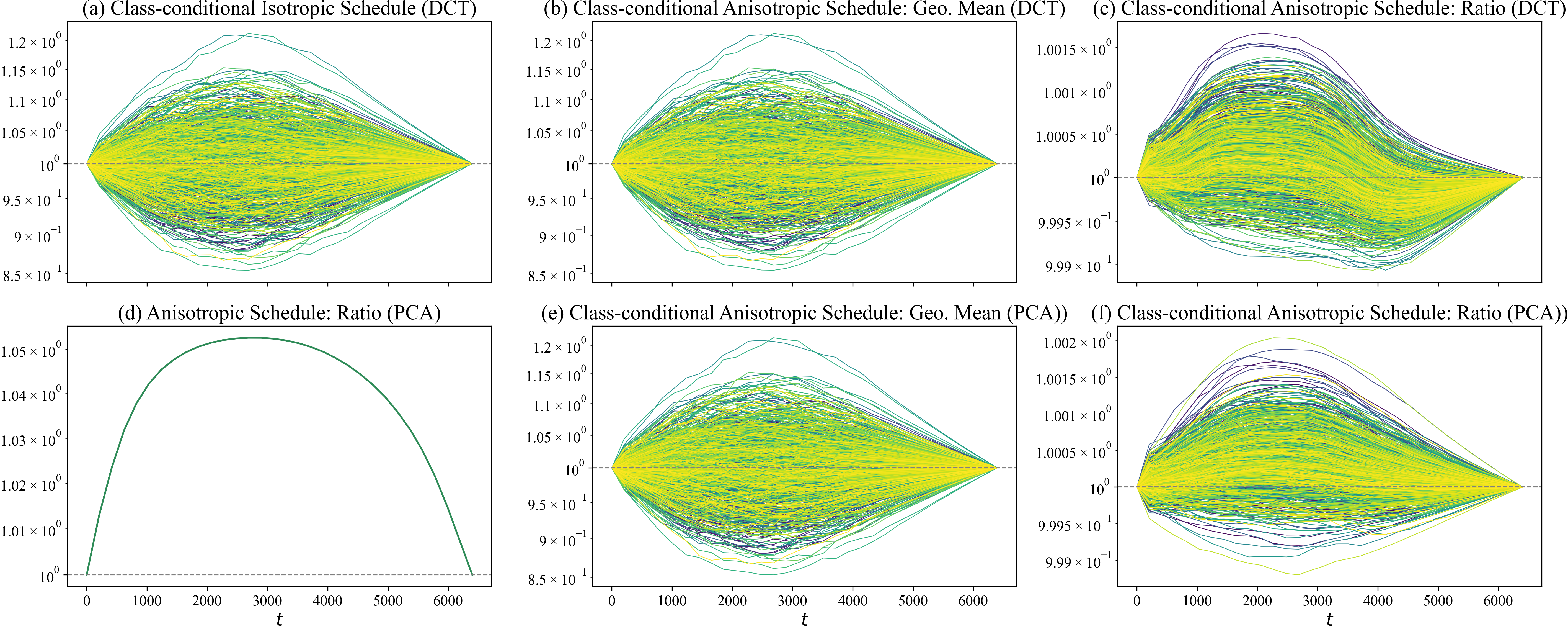}
    \caption{\textbf{ImageNet-64 learned schedule analysis.}
(a) Class-conditional isotropic schedules shown as $g_y(t)/\bar g(t)$, where $\bar g(t)$ denotes the geometric mean across classes.
(b) Class-conditional anisotropic schedules summarized by the geometric mean $\sqrt{g^{\mathrm{DCT}}_{1,y}(t)g^{\mathrm{DCT}}_{2,y}(t)}$ and normalized by its class-wise geometric mean.
(c) Class-conditional anisotropy ratios $g^{\mathrm{DCT}}_{1,y}/g^{\mathrm{DCT}}_{2,y}$.
(d) PCA-based (class-conditional basis) anisotropic schedule ratio $g^{\mathrm{PCA}}_{1,y}(t)/g^{\mathrm{PCA}}_{2,y}(t)$.
(e) PCA-based (class-conditional basis) class-conditional anisotropic schedules shown as $\sqrt{g^{\mathrm{PCA}}_{1,y}(t)g^{\mathrm{PCA}}_{2,y}(t)}$ normalized by the geometric mean across classes.
(f) PCA-based (class-conditional basis) class-conditional anisotropy ratios $g^{\mathrm{PCA}}_{1,y}(t)/g^{\mathrm{PCA}}_{2,y}(t)$.
All ratios are plotted on a logarithmic $y$-axis, and dashed horizontal lines indicate the reference value $1$.}
    \label{fig:imagenet_schedule}
\end{figure*}

\subsection{Training and Inference Speed}
\label{ss:speed_analysis}

We report wall-clock training and inference times for our method versus the EDM
baseline on a single NVIDIA H200 GPU (141\,GB HBM3e memory, $\sim$4.8\,TB/s
memory bandwidth). The large memory capacity
is useful here because our schedule-gradient estimator
(Theorem~\ref{t:del_theta_score}) materializes higher-order directional
derivatives on top of the standard score-network forward/backward, which
increases peak memory pressure relative to vanilla EDM training.

\paragraph{Training.}
Our method increases wall-clock training time relative to EDM under the matched
1.2M-image fine-tuning protocol; the slowdown comes mainly from the schedule
optimization step, which adds higher-order directional-derivative computations
and the corresponding backward passes. Per-dataset numbers are reported in
Table~\ref{tab:speed} (top).

\paragraph{Inference.}
At inference time the per-step overhead is small. The dominant cost remains
score-network evaluations; the additional per-step work is structured matrix
operations involving $M_t(\theta)$ and related terms. Under the projector-based
parameterizations of Section~\ref{s:implementation_details}, these reduce to
inexpensive subspace-wise scalings and avoid explicit dense $d{\times}d$
matrix square roots. Per-dataset numbers (50k samples, batch size 64,
NFE $=35$) are reported in Table~\ref{tab:speed} (bottom).

\begin{table}[ht]
\centering
\small
\setlength{\tabcolsep}{8pt}
\begin{tabular}{l|cc}
\toprule
\multicolumn{3}{c}{\textbf{Training} (1.2M-image fine-tuning, single H200)} \\
\midrule
Dataset     & EDM fine-tuning time & Our fine-tuning time \\
\midrule
CIFAR-10    & 58m    & 2h 20m \\
AFHQv2      & 1h 53m & 8h     \\
FFHQ        & 1h 53m & 8h     \\
ImageNet-64 & 5h 04m & 35h    \\
\midrule
\multicolumn{3}{c}{\textbf{Inference} (50k samples, batch size 64, NFE $=35$, single H200)} \\
\midrule
Dataset     & EDM  & Ours \\
\midrule
CIFAR-10    & 14m  & 14m  \\
AFHQv2      & 25m  & 27m  \\
FFHQ        & 25m  & 27m  \\
ImageNet-64 & 49m  & 58m  \\
\bottomrule
\end{tabular}
\vspace{1em}
\caption{Wall-clock training and inference times on a single NVIDIA H200 GPU.
\textbf{Top:} fine-tuning from the released EDM checkpoint for 1.2M additional
images. \textbf{Bottom:} sampling 50k images at NFE $=35$ with batch size 64.
Inference cost is essentially unchanged from EDM, while training cost is
2--7$\times$ higher due to the schedule-gradient estimator.}
\label{tab:speed}
\end{table}

\subsection{Fine-Tuning Controls with Matched Image Budgets}
\label{ss:edm_same_compute}

A natural concern is whether our consistent FID improvement over EDM in
Table~\ref{tab:fid-all} could be partially attributable to the extra 1.2M-image
fine-tuning consumed by our schedule-learning protocol
(Section~\ref{ss:training_setup}), rather than to the learned anisotropic
schedule itself. To rule this out, we run two complementary checks on
\textbf{ImageNet-64}, the largest and most challenging of our benchmarks, and
therefore the dataset where additional training has the greatest a priori
chance of further reducing EDM's FID. The other datasets are smaller and
closer to the convergence regime, so the conclusions transfer \emph{a fortiori}.

\paragraph{(a) Fine-tuning EDM with its own recipe.}
We fine-tune the released EDM checkpoint for 1.2M further images using EDM's
own intended training recipe (isotropic schedule, original optimizer, original
loss weighting), and re-evaluate FID under the standard EDM sampling protocol.
As shown in Table~\ref{tab:edm_extra_training}, the extended-training EDM
achieves FID values within run-to-run variance of the original EDM numbers
reproduced in Table~\ref{tab:fid-all-error-bar}, confirming that the released
checkpoint is already well converged.

\begin{table}[ht]
\centering
\small
\setlength{\tabcolsep}{6pt}
\begin{tabular}{l|cccccccc}
\toprule
\multicolumn{9}{c}{\textbf{ImageNet-64} (EDM + 1.2M extra fine-tuning images)} \\
\midrule
NFE & 9 & 11 & 13 & 15 & 19 & 39 & 79 & 119 \\
\midrule
FID
& \fidstd{35.39}{0.1752} & \fidstd{15.22}{0.0715} & \fidstd{8.182}{0.0149} & \fidstd{5.473}{0.0373}
& \fidstd{3.594}{0.0427} & \fidstd{2.430}{0.0365} & \fidstd{2.297}{0.0286} & \fidstd{2.278}{0.0264} \\
\bottomrule
\end{tabular}
\vspace{1em}
\caption{EDM fine-tuning control on ImageNet-64 with a matched image budget. We fine-tune the
released EDM checkpoint for an additional 1.2M images using EDM's own training
recipe and re-evaluate FID across the same NFE budgets as
Table~\ref{tab:fid-all}. The values are within run-to-run variance of the
original EDM numbers (cf.\ Table~\ref{tab:fid-all-error-bar}, ImageNet-64 row),
indicating that the released checkpoint is already well converged and that
the improvements reported in our main results stem from the learned
anisotropic schedule rather than from additional compute. Each FID is the
minimum across 3 independent generations of 50k images; standard deviations
are shown in parentheses.}
\label{tab:edm_extra_training}
\end{table}

\paragraph{(b) Fine-tuning EDM with our hyperparameters.}
A complementary concern is that our specific training hyperparameters
(optimizer state, learning-rate schedule, EMA decay) might themselves be
responsible for the improvement. We therefore fine-tune EDM for 1.2M images
using \emph{our} hyperparameters but with the schedule held fixed at EDM's
isotropic baseline, and re-evaluate FID. As shown in
Table~\ref{tab:edm_our_hyperparams}, the values are again within run-to-run
variance of the original EDM, ruling out hyperparameter-induced gains.

\begin{table}[ht]
\centering
\small
\setlength{\tabcolsep}{6pt}
\begin{tabular}{l|cccccccc}
\toprule
\multicolumn{9}{c}{\textbf{ImageNet-64} (EDM + 1.2M images, our hyperparameters, isotropic schedule)} \\
\midrule
NFE & 9 & 11 & 13 & 15 & 19 & 39 & 79 & 119 \\
\midrule
FID
& \fidstd{35.35}{0.1767} & \fidstd{15.19}{0.0714} & \fidstd{8.164}{0.0139} & \fidstd{5.456}{0.0383}
& \fidstd{3.585}{0.0433} & \fidstd{2.428}{0.0365} & \fidstd{2.294}{0.0289} & \fidstd{2.276}{0.0261} \\
\bottomrule
\end{tabular}
\vspace{1em}
\caption{Hyperparameter-matched EDM sanity check on ImageNet-64. We fine-tune
the released EDM checkpoint for 1.2M images using \emph{our} training
hyperparameters (optimizer, learning-rate schedule, EMA, batch size) while
keeping the schedule isotropic, and re-evaluate FID. Values are within
run-to-run variance of both the original EDM numbers
(Table~\ref{tab:fid-all-error-bar}, ImageNet-64 row) and the EDM-recipe
sanity check (Table~\ref{tab:edm_extra_training}), ruling out
hyperparameter-induced gains. Each FID is the minimum across 3 independent
generations of 50k images; standard deviations are shown in parentheses.}
\label{tab:edm_our_hyperparams}
\end{table}

% =====================================================================
\subsection{Subspace Bases}
\label{s:bases}
% =====================================================================
We use two families of orthonormal subspace bases throughout the paper:
a fixed structured basis (2D-DCT) and a data-adaptive class-conditional basis
obtained from PCA. This section defines the 2D-DCT basis (Section~\ref{ss:2d_dct_appendix})
and analyzes the geometric structure of the learned PCA bases (Section~\ref{ss:pca_geometry}).

\subsubsection{2D-DCT Bases}
\label{ss:2d_dct_appendix}
Let $H$ be the image side length and $d = H^2$.
The two-dimensional DCT (type-II) basis over $\mathbb{R}^{H \times H}$ is defined as follows.
For each pair $(p,q)\in\{0,\dots,H-1\}^2$, the basis is
\[
\begin{aligned}
\Phi_{p,q}(x,y)
&= \gamma_p \gamma_q
\cos\!\Big(\tfrac{(2x+1)p\pi}{2H}\Big)
\cos\!\Big(\tfrac{(2y+1)q\pi}{2H}\Big), \\
&\qquad x,y=0,\dots,H-1 .
\end{aligned}
\]
with normalization factors
\[
\gamma_p =
\begin{cases}
H^{-1/2}, & p=0, \\[6pt]
\sqrt{2}H^{-1/2}, & p>0,
\end{cases}
\qquad
\gamma_q =
\begin{cases}
H^{-1/2}, & q=0, \\[6pt]
\sqrt{2}H^{-1/2}, & q>0.
\end{cases}
\]
Vectorizing each $\Phi_{p,q}$ into $\mathbb{R}^d$ and enumerating them yields the orthonormal basis $\{v_1,\dots,v_{H^2}\}$, which are the 2D-DCT basis of $\mathbb{R}^d$. Please refer to Figure~\ref{fig:2d_dct_ex} for example 2D-DCT bases.

\begin{figure*}[ht]
    \centering
    \includegraphics[width=1\linewidth]{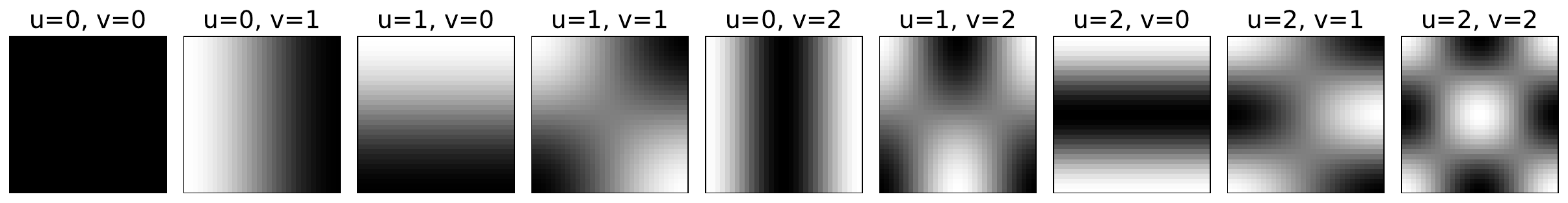}
    \vspace{-18pt}
    \caption{The first nine 2D-DCT bases ordered by increasing frequency.}
    \label{fig:2d_dct_ex}
\end{figure*}

\subsubsection{PCA Bases}
\label{ss:pca_geometry}

To analyze the geometry of the learned PCA subspaces, we first visualize the leading principal components for representative classes (cat, dog, airplane) in Figure~\ref{fig:basis_compare}. Each row shows the top principal directions reshaped into spatial patterns. We observe that semantically related classes (e.g., cat and dog) exhibit qualitatively similar dominant spatial structures, whereas structurally distinct classes (e.g., airplane) show noticeably different patterns.

\begin{figure}[ht]
    \centering
    \includegraphics[width=0.9\linewidth]{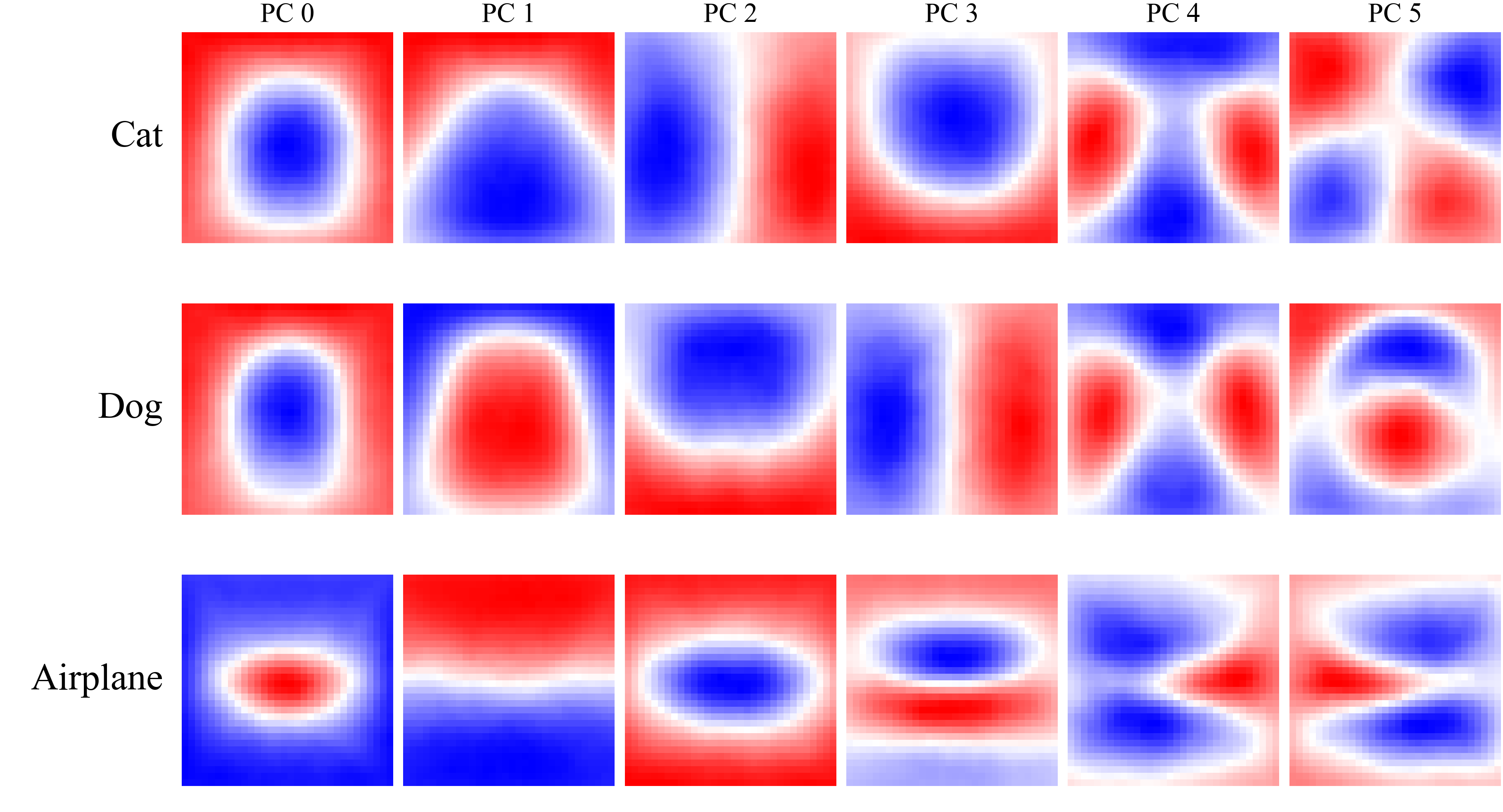}
    \caption{\textbf{Class-specific PCA subspaces.}
    Leading principal components for representative CIFAR-10 classes (cat, dog, airplane), reshaped as spatial patterns.
    Semantically related classes (e.g., cat and dog) exhibit qualitatively similar dominant structures, whereas visually distinct classes (e.g., airplane) display different principal directions.}
    \label{fig:basis_compare}
\end{figure}

To quantify cross-class similarity, we compute pairwise distances between subspaces using the Frobenius norm of projector differences, i.e.,
$$
d\left(Q_1, Q_2\right)=\frac{\left\|Q_1 Q_1^{\top}-Q_2 Q_2^{\top}\right\|_F}{\sqrt{2 k}},
$$
where $Q_1, Q_2$ are orthonormal bases.
We then apply classical multidimensional scaling (MDS) to embed this distance matrix into two dimensions for visualization. MDS constructs a low-dimensional embedding whose pairwise Euclidean distances best approximate the original subspace distances. Figure~\ref{fig:basis_mds} reveals clear semantic clustering, with animal classes grouped together and separated from vehicle classes.

These observations suggest that the learned PCA subspaces capture semantically meaningful geometric structure.

\begin{figure}[ht]
    \centering
    \includegraphics[width=0.65\linewidth]{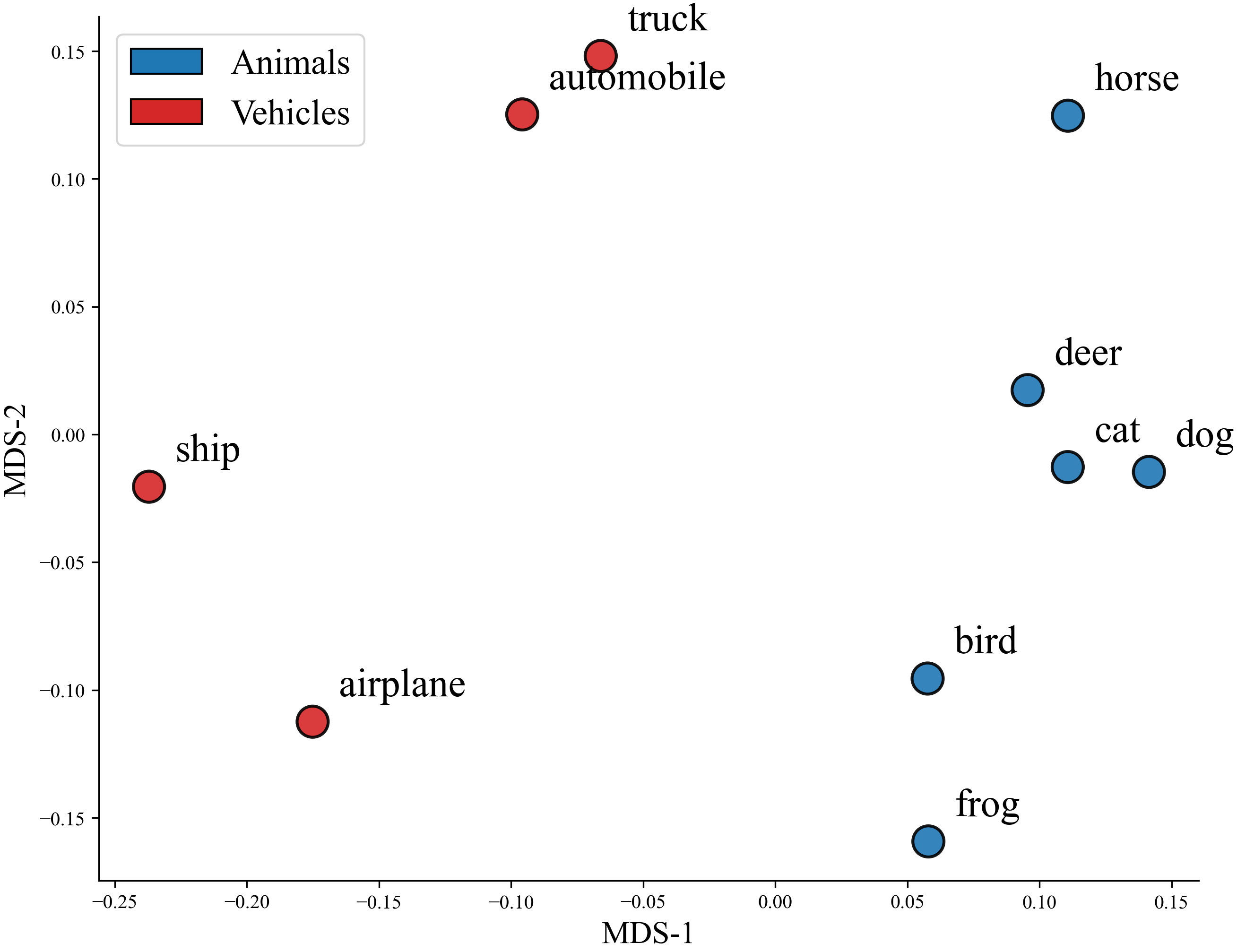}
    \caption{\textbf{Subspace distance embedding via classical MDS.}
    Pairwise distances between class-specific PCA subspaces are computed using the Frobenius norm between orthogonal projectors.
    Classical multidimensional scaling (MDS) embeds the resulting distance matrix into two dimensions.
    Animal classes cluster together and are separated from vehicle classes, indicating that the learned subspaces capture semantically structured geometry.}
    \label{fig:basis_mds}
\end{figure}

% =====================================================================
\section{Broader Impacts}
\label{ss:broader_impacts}
% =====================================================================
This work is methodological: it improves sample quality and few-step generation in diffusion models. Positive impacts include reduced compute and energy cost for high-quality generation, particularly in the low-NFE regime. Dual-use concerns common to all generative-model improvements apply, i.e., higher-fidelity samplers can be misused for deepfakes, non-consensual imagery, or disinformation. Our experiments use only publicly available datasets and fine-tune released EDM checkpoints; we release no new datasets or models trained on novel sources.

%%%%%%%%%%%%%%%%%%%%%%%%%%%%%%%%%%%%%%%%%%%%%%%%%%%%%%%%%%%%

\end{document}